\documentclass[journal]{IEEEtran}

\ifCLASSINFOpdf
\else
   \usepackage[dvips]{graphicx}
\fi
\usepackage{url}
\usepackage{graphicx}

\usepackage[noadjust]{cite}
\usepackage{amsmath,amsthm,graphicx,cite,booktabs,amssymb,gensymb, color, xcolor, fancyhdr,float,perpage,tikz,gensymb, colortbl,enumerate,  braket}
\usepackage[linesnumbered,vlined,ruled]{algorithm2e}
\usepackage{algorithmic,array,enumerate,rotating}
\usepackage{subfigure,epsfig,dblfloatfix}
\usepackage{multirow,multicol}
\newtheorem{lemma}{Lemma}
\newtheorem{theorem}{Theorem}
\newtheorem{corollary}{Corollary}
\usepackage{paralist}
\usepackage{hyperref}
\usetikzlibrary{matrix}
\hypersetup{
	colorlinks   = true,
	linkcolor    = red,
	citecolor    = red
}

\begin{document}

\title{Weight-Decay Turns Transformer Loss Landscapes Villani: Functional-Analytic Foundations for Optimization and Generalization}

\author{Abhijit Das, and Sayantan Dutta, \IEEEmembership{Member, IEEE},

\thanks{A. Das and S. Dutta are with the Science and Technology Organization, GE HealthCare, Bangalore 560066, Karnataka, India. (E-mail: sayantan.dutta1@gehealthcare.com)}

\thanks{Corresponding author: Sayantan Dutta.}

\thanks{A. Das and S. Dutta contributed equally to this work.}

}

%\markboth{Submitted to IEEE Transactions on Pattern Analysis and Machine Intelligence} {Das \MakeLowercase{\textit{et al.}}: Weight-Decay Turns Transformer Loss Landscapes Villani: Functional-Analytic Foundations for Optimisation and Generalisation}
\markboth{Das and Dutta: Weight-Decay Turns Transformer Loss Landscapes Villani}{}

\maketitle

%==================== ABSTRACT ==================================

\begin{abstract}

Weight decay is widely used as a regularizer in large language models, yet its precise role in shaping Transformer loss landscapes remains theoretically underexplored. This paper provides the first rigorous functional-analytic characterization of the standard Transformer objective—cross-entropy loss with $L^2$ regularization—by proving it satisfies Villani's criteria for coercive energy functions. Specifically, we show that the regularized loss $\mathcal{F}$ is infinitely differentiable, grows at least quadratically, has Gaussian-integrable tails, and satisfies the differential growth condition $-\Delta\mathcal{F} + \tfrac{1}{s}\|\nabla\mathcal{F}\|^{2} \to \infty$ as $\|\theta\| \to \infty$ for all $s>0$. From this structure, we derive explicit log-Sobolev and Poincaré constants $C_{\mathrm{LS}} \leq \lambda^{-1} + d/\lambda^{2}$, linking the regularization strength $\lambda$ and model dimension $d$ to finite-time convergence guarantees for noisy stochastic gradient descent and PAC-Bayesian generalization bounds that tighten with increasing $\lambda$. To validate our theory, we introduce a scalable Villani diagnostic $\Psi_s(\theta) = -\Delta \mathcal{F} + s^{-1}\|\nabla \mathcal{F}\|^2$ and estimate it efficiently using Hutchinson trace probes in models with over 100M parameters. Experiments on GPT-Neo-125M across Penn Treebank and WikiText-103 confirm the predicted quadratic growth of $\Psi_s$, spectral inflation of the Hessian, and exponential convergence behavior consistent with our log-Sobolev analysis. These results demonstrate that weight decay not only improves generalization empirically but also establishes the mathematical conditions required for fast Langevin mixing and theoretically grounded curvature-aware optimization in deep learning.
\end{abstract}

\begin{IEEEkeywords}
Transformer Optimization, Weight Decay Regularization, Villani Function, Barren Plateaus, Log-Sobolev Inequality, Langevin Dynamics, PAC-Bayes, Curvature-aware Training.
\end{IEEEkeywords}

\IEEEpeerreviewmaketitle

%======================= INTRODUCTION ==================================

\section{Introduction}
\label{sec:Introduction}

Transformers have revolutionised sequence modelling across natural-language processing \cite{Zhou2025learning, Liu2025Graph}, vision \cite{Park2025deformable, Tang2025NASPED}, and speech \cite{vaswani2017attention, dosovitskiy2021image, gulati2020conformer} by replacing recurrent mechanics with multi-head self-attention and deep feed-forward blocks. Despite their empirical dominance, the theoretical underpinnings of optimization in deep attention networks remain poorly understood, particularly in comparison to classical architectures. Concurrently, optimal transport theory provides a principled framework through \emph{Villani functions}—smooth, coercive energies that ensure log-Sobolev and Poincaré inequalities for the Gibbs distribution $e^{-f(x)}$, where $x\in\mathbb{R}^{N}$ denotes the input and $f: \mathbb{R}^d \to \mathbb{R}$ is a smooth scalar field, enabling exponential convergence guarantees for Langevin dynamics \cite{ambrosio2003optimal}. Establishing that a practical deep-learning objective belongs to this class would unlock powerful analytic tools for understanding both optimization and generalization.

Practically, Transformer training augments token-level cross-entropy with $L^2$ weight decay. Although often viewed as a heuristic for regularization \cite{krogh1992simple, loshchilov2017decoupled}, recent evidence suggests that this quadratic term plays a geometric role, enforcing the curvature necessary for functional inequalities \cite{nitanda2025propagation, chizat2022mean}. However, a rigorous verification that full-scale transformers meet these criteria is lacking. This paper closes that gap.

\smallskip

\noindent\emph{Problem Statement:} Transformer training involves minimizing a cross-entropy objective augmented with a $L^2$ weight-decay penalty. Specifically, we study the regularized loss:

\begin{equation}
\mathcal{F}(\theta) = \frac{1}{N}\sum_{i=1}^{N}\!\bigl[-\log p_\theta(y_i \mid x_i)\bigr] + \frac{\lambda}{2}\,\lVert\theta\rVert^{2},
\quad \lambda>0,
\label{eq:reg_loss_intro}
\end{equation}
where $\lambda$ be the weight-decay factor and \(p_\theta(y_i \mid x_i)\) denotes the soft-max probability assigned to the ground-truth token $y_i\in\mathbb{R}$ given the input sequence $x_i\in\mathbb{R}$, and $\theta\in\mathbb{R}^{d}$ aggregates all model parameters. The central question is:
\begin{quote}
\emph{Does \( \mathcal{F}(\theta) \) satisfy Villani's coercivity, integrability, and differential growth conditions, such that the corresponding log-Sobolev and Poincaré constants depend only on \( \lambda \) and \( d \), not on the data distribution?}
\end{quote}
An affirmative answer would establish a functional-analytic foundation for Transformer optimization, offering insight into the empirical efficiency of Langevin-type training schemes (\textit{e.g.}, SGD or Adam augmented with thermal noise) observed in large-scale language models \cite{welling2011bayesian, neelakantan2015adding, Huang2021Stochastic}.

\noindent\emph{Prior Work and Challenges:} Villani–type coercivity has thus far been rigorously verified only for objectives that are either \emph{convex} in the weights or possess a fixed, data-independent Jacobian. This includes models such as logistic regression, kernel machines, and two-layer ReLU networks, whose gradients are globally Lipschitz and whose Hessians exhibit quadratic growth once an $L^2$ penalty is applied—conditions sufficient to satisfy Villani’s criteria \cite{nitanda2025propagation,chizat2022mean}. Transformers, by contrast, introduce data-dependent Jacobians through their architecture. Specifically, the multi-head self-attention mechanism transforms activations via $softmax(QK^\top/\sqrt{d})$, where $Q$ and $K$ are the query and key matrices. Furthermore, LayerNorm modules dynamically rescale activations based on the input distribution at each layer. These nonlinearities break the spectral assumptions leveraged in shallow-network proofs, rendering existing coercivity results inapplicable beyond depth-2 feed-forward architectures.

Extensive ablation studies have shown that increasing weight decay (i) \emph{sharpens} the upper tail of the Hessian spectrum, (ii) lowers the rank of attention matrices, and (iii) accelerates the decay of gradient-norm distributions \cite{Kobayashi2024Weight, li2023how, Noci2023Shaped, ghorbani2022transformer}. These empirical trends suggest that quadratic penalties may enforce stronger curvature ``at infinity''. However, such findings remain largely phenomenological: no prior study has quantitatively evaluated the Villani diagnostic, $\Psi_s = -\Delta\mathcal{F} + s^{-1}\|\nabla\mathcal{F}\|^{2}$, where $s$ is the temperature parameter and $\Delta$, $\nabla$ denote the Euclidean Laplacian and gradient, respectively. Furthermore, existing literature has not established a formal connection between these observed spectral patterns and the explicit log-Sobolev or Poincaré constants.

This reveals a fundamental disconnect: existing \emph{theoretical guarantees} apply only to shallow or linear models \cite{chizat2022mean, mei2018mean}, while existing \emph{empirical findings} document meaningful curvature behaviors in deep attention-based networks \cite{Kobayashi2024Weight}. Bridging this gap—by proving Villani coercivity and validating practical diagnostics in full-scale Transformers—is the central goal of this work.

The key contributions of this work are fourfold:    
\begin{enumerate}
\item We prove that $\mathcal{F}(\theta)$ is a Villani function under bounded input embeddings—a natural consequence of standard tokenisation schemes. The quadratic penalty is \emph{necessary and sufficient} for meeting Villani's differential growth conditions.

\item We introduce empirical diagnostics (Sec.~\ref{sec:GeoInter}) to visualize the divergence of $\Psi_s(\theta)$, linking weight decay to curvature at infinity.

\item We derive explicit convergence rates for Langevin-based optimizers using log-Sobolev theory, explaining observed gains in training efficiency and generalization, and unifying recent particle-descent analyses \cite{vollmer2015non, Raginsky2017NonConvex} with weight-decay practice (Sec.~\ref{sec:Optimization}).

\item We provide a reproducible experimental suite (Sec.~\ref{sec:Experiment}) for evaluating functional-analytic properties in large Transformer models.
            
\end{enumerate}

The rest of this paper is organized as follows. Sec.~\ref{sec:TheoPreli} introduces the Villani framework and formalizes the Transformer loss as a Gibbs energy. Sec.~\ref{sec:Methodology} presents the theoretical proof, followed by geometric diagnostics (Sec.~\ref{sec:GeoInter}) and optimization consequences (Sec.~\ref{sec:Optimization}). Empirical validation appears in Sec.~\ref{sec:Experiment}, with discussion and future directions in Secs.~\ref{sec:Discussion}–\ref{sec:Conclusions}.

%=============================== PRELIMINARIES ========================

\section{Theoretical Preliminaries}
\label{sec:TheoPreli}

This section introduces the foundational concepts required for our theoretical results: parameter notation, Transformer forward maps, and the Villani function framework that underpins our analysis in Secs.~\ref{sec:Methodology}–\ref{sec:Optimization}.

\subsection{Notation and Basic Operators}
\label{subsec:notation}

We consolidate all trainable parameters of an $L$-layer Transformer (see Fig.~\ref{fig:transformer_block})—including weight matrices $\{W_\ell\}_{\ell=1}^L$ (query, key, value, feed-forward, etc.) and their corresponding bias vectors $\{b_\ell\}_{\ell=1}^L$—into a vector:
\begin{equation}
\theta = \bigl( W_{1},\dots,W_{L},\, b_{1},\dots,b_{L} \bigr) \in \mathbb{R}^{d},
\end{equation}
where the total parameter count is given by:
\begin{equation}
d = \sum_{\ell=1}^{L} \bigl( \dim W_{\ell} + \dim b_{\ell} \bigr).
\end{equation}
We use $\lVert \cdot \rVert$ and $\lVert \cdot \rVert_F$ for the Euclidean and Frobenius norms.

%We denote Euclidean and Frobenius norms by $\lVert \cdot \rVert$ and $\lVert \cdot \rVert_F$.

%--------------------------------------------------------------------
    
\subsection{Transformer Forward Map in Measure Form}
\label{subsec:forward-map}

Given a token sequence \( x = (x_{1}, \dots, x_{N}) \) from a vocabulary of size \( V \), the Transformer computes logits \( z_{\theta}(x) \in \mathbb{R}^V \) through a composition of multi-head attention, residual paths, LayerNorm, and position-wise feed-forward layers. As each component is locally Lipschitz, the mapping \( \theta \mapsto z_{\theta}(x) \) is $C^\infty$ for every fixed $x$. This smoothness property allows analysis of higher-order derivatives and Hessian behavior (Sec.~\ref{subsec:villani}).
%allowing higher-order differential analysis and Hessian behavior (Sec.~\ref{subsec:villani}).
Next-token probabilities follow the softmax rule:
\begin{equation}
p_{\theta}(y\mid x) = \frac{\exp\bigl(z_{y,\theta}(x)\bigr)} {\sum_{v=1}^{V}\exp\bigl(z_{v,\theta}(x)\bigr)}.
\label{eq:softmax}
\end{equation}

This formulation links the probabilistic training objective (cross-entropy loss) with the analytic structure (gradients and Hessians) essential to our theoretical analysis. The smoothness of the forward map enables differential-geometric reasoning, while the softmax parametrization supplies the probabilistic foundation for functional inequalities.

We assume token embeddings are bounded, i.e., $\lVert x_{i} \rVert \le B$ for some constant $B > 0$. This mild condition, satisfied by common tokenization methods such as lookup tables and byte-pair encoding, ensures that the quadratic penalty dominates the loss at large $\|\theta\|$—an essential requirement for verifying Villani conditions.

\subsection{Regularised Loss as an Energy Function}
\label{subsec:energy}

Given the dataset $\mathcal{D} = \{(x_i, y_i)\}_{i=1}^N$, we define the standard training objective as:
\begin{equation}
\mathcal{F}(\theta) = \underbrace{\frac{1}{N} \sum_{i=1}^{N} [-\log p_{\theta}(y_i \mid x_i)]}_{\mathcal{L}(\theta)} + \underbrace{\frac{\lambda}{2} \lVert \theta \rVert^2}_{\mathcal{R}(\theta)}, \quad \lambda > 0.
\label{eq:regularised_loss}
\end{equation}

The data-fitting term $\mathcal{L}(\theta)$ represents the average cross-entropy loss, measuring how well the network explains the observed token sequences. This term inherits the smoothness properties from the forward map $\theta \mapsto z_{\theta}(x)$. Furthermore, $\mathcal{R}(\theta)$ corresponds to the standard $L^2$ weight decay. While commonly viewed as a regularizer, we show $\mathcal{R}(\theta)$ plays a critical analytic role: without its quadratic growth, the loss lacks the curvature needed for Villani's differential conditions.
The bounded input assumption ($\lVert x_i \rVert \le B$) ensures $\mathcal{L}(\theta)$ grows at most linearly with $\lVert \theta \rVert$, allowing the quadratic penalty to dominate asymptotically.

From a thermodynamic viewpoint, $\mathcal{F}(\theta)$ defines a free energy functional, where $\mathcal{L}(\theta)$ acts as internal energy determined by the data distribution, while $\mathcal{R}(\theta)$ serves as a confining potential with strength controlled by $\lambda$.
This interpretation proves particularly valuable when analyzing the associated Gibbs measure $\exp(-\beta\mathcal{F}(\theta))$, where $\beta > 0$ is the inverse temperature parameter, whose log-Sobolev and Poincaré constants depend exclusively on $\lambda$ and $d$, remaining independent of the underlying data distribution.

%The associated Gibbs measure \( \exp(-\beta \mathcal{F}(\theta)) \) admits log-Sobolev and Poincaré inequalities with constants dependent only on $\lambda$ and $d$.

\subsection{Villani Functions}
\label{subsec:villani}

The Villani framework offers a rigorous pathway to connect geometric properties of loss functions with the convergence behavior of Langevin dynamics \cite{ambrosio2003optimal}. In our setting, this theory provides the analytic foundation for proving curvature-driven mixing guarantees under weight decay.
A smooth function $f:\mathbb{R}^{d}\to\mathbb{R}$ with a temperature parameter $s>0$ is called a \emph{Villani function} if it satisfies the following three conditions:
\begin{enumerate}[(i)]
\item \emph{Coercivity at infinity:}
\begin{equation}
\lim_{\lVert x\rVert\to\infty} f(x) = \infty,
\label{eq:villani_coercive}
\end{equation}

\item \emph{Gaussian–tail integrability:}
\begin{equation}
\int_{\mathbb{R}^{d}} e^{-\tfrac{2}{s}\,f(x)}\,dx <\infty, ~\forall s>0,
\label{eq:villani_integrable}
\end{equation}

\item \emph{Differential growth:}
\begin{equation}
\lim_{\lVert x\rVert\to\infty} \Bigl[ -\Delta f(x) + \frac{1}{s}\,\lVert\nabla f(x)\rVert^{2} \Bigr] = \infty, ~\forall s>0.
\label{eq:villani_diff_growth}
\end{equation}
\end{enumerate}
The three Villani conditions capture distinct geometric and probabilistic properties critical for analyzing stochastic dynamics. The first condition~\eqref{eq:villani_coercive} ensures that the energy landscape grows unbounded at infinity, preventing trajectories from diverging. The second condition~\eqref{eq:villani_integrable} guarantees that the Gibbs measure has Gaussian-like tail decay, ensuring proper normalization. The third and most technical condition~\eqref{eq:villani_diff_growth} ensures that curvature dominates gradient drift in the asymptotic regime by requiring a positive-definite trade-off between second-order dissipation and first-order gradient energy.

These conditions collectively ensure strong functional inequalities for the associated Gibbs measure $\mu_{f}(dx)=Z^{-1}e^{-f(x)}\,dx$, where $Z = \int_{\mathbb{R}^d} e^{-f(x)}\, dx$ is the partition function. In particular, the measure $\mu_{f}$ satisfies \emph{dimension-free} log-Sobolev and Poincaré inequalities, with constants determined entirely by the curvature parameters in \eqref{eq:villani_diff_growth}. This leads to robust convergence guarantees: classical results \cite{Gross1975, chafai2024log} show that overdamped Langevin dynamics governed by $f$ exhibits exponential $L^{2}$-mixing:
\begin{equation}
\mathrm{KL}\!\bigl(\mathcal{L}(\Theta_{t})\,\|\,\mu_{f}\bigr) \le  e^{-2t/C_{\mathrm{LS}}(f)}\;
\mathrm{KL}\!\bigl(\mathcal{L}(\Theta_{0})\,\|\,\mu_{f}\bigr),
\label{eq:entropy_decay}
\end{equation}
where \(C_{\mathrm{LS}}(f)\) is the \emph{optimal log–Sobolev constant}\footnote{In this context, $C_{\mathrm{LS}}(f)$ denotes the smallest constant $c>0$ such that the logarithmic Sobolev inequality $\operatorname{Ent}_{\mu_{f}}(g^2)\le 2c\,\int\lVert\nabla g\rVert^{2}\,d\mu_{f}, ~~\forall g$, where $\operatorname{Ent}_{\mu_f}(g^2)$ is the entropy of $g^2$ with respect to the probability measure $\mu_f$.}, \(\mathrm{KL}\) denotes the Kullback–Leibler divergence, and the convergence rate is independent of the initial condition \(\Theta_0\).

We show in Sec.~\ref{sec:Methodology} that $\mathcal{F}(\theta)$ satisfies all three Villani conditions under bounded embeddings, justifying the use of Langevin-based analysis for Transformers.

\subsection{Functional Inequalities and Optimisation Dynamics}
\label{subsec:functional-ineq}

Villani's conditions offer not only theoretical guarantees but also practical insights into the convergence behavior of Gaussian-noise-based optimization algorithms used in modern machine learning.
For the overdamped Langevin stochastic differential equation with inverse temperature $\beta > 0$:
\begin{equation}
d\Theta_{t} = -\nabla f(\Theta_{t})\,dt + \sqrt{\frac{2}{\beta}}\,dB_{t},
\label{eq:langevin}
\end{equation}
where $B_t$ is standard Brownian motion in $\mathbb{R}^d$, bridging continuous-time diffusions and discrete optimization algorithms commonly used in machine learning.

The entropy contraction property stems from the Bakry–Émery criterion \cite{bakry1985diffusions, guionnet2004logsobolev}, which shows that the relative entropy between the distribution of $\Theta_t$ and the Gibbs measure $\mu_f$ decays exponentially, as given in eq.~\eqref{eq:entropy_decay}. This decay is governed by the log-Sobolev constant \(C_{\mathrm{LS}}(f)\), which controls the system's intrinsic mixing time: smaller constants yield faster convergence to equilibrium, while larger constants imply slower diffusion and exploration.

This continuous-time Langevin behavior translates into practical optimization schemes via time discretization. This discretization–diffusion link has also been studied in the mean-field and Wasserstein gradient flow literature \cite{mei2018meanfield, sirignano2020meanfield}. Specifically, discretizing eq.~\eqref{eq:langevin} with step size \(\eta \ll 1\) yields stochastic gradient Langevin dynamics (SGLD), and more generally encompasses noisy versions of SGD or Adam. These algorithms thus inherit the convergence behavior of their continuous counterparts, up to a discretization error.

Importantly, this discretization retains key convergence guarantees, with the log-Sobolev constant $C_{\mathrm{LS}}(f)$ governing the \emph{finite–time convergence bound} of the resulting discrete iterates. This insight—formalized by Raginsky et al.~\cite{Raginsky2017NonConvex} and later refined in \cite{xu2017global}—connects the geometry of the loss function to the dynamics of noisy optimizers, offering theoretical justification for observed training behavior.

In our setting, taking \(f = \mathcal{F}\) allows us to express these bounds entirely in terms of the weight-decay coefficient \(\lambda\) and parameter dimension \(d\). This substitution bridges functional inequalities and neural optimization, leading to concrete generalization and convergence guarantees for Transformer training. Detailed derivations under realistic mini-batch noise are provided in Sec.~\ref{sec:Optimization}.

%A key payoff of Villani’s conditions is the control they provide over sampling and optimisation algorithms that inject Gaussian noise. For any inverse temperature \(\beta>0\) consider the overdamped Langevin stochastic differential equation
%\begin{equation}
%d\Theta_{t} = -\nabla f(\Theta_{t})\,dt + \sqrt{\frac{2}{\beta}}\,dW_{t},
%\label{eq:langevin} \end{equation}
%where $W_{t}$ is a standard Brownian motion in $\mathbb{R}^{d}$.

%\emph{Entropy contraction:}
%Classical Bakry–\'Emery theory implies that the relative entropy between the law of \(\Theta_{t}\) and the stationary measure $\mu_{f}$ decays exponentially, satisfying \eqref{eq:entropy_decay}. Hence \(C_{\mathrm{LS}}(f)\) sets the intrinsic “mixing time’’ scale.

%\emph{From SDE to SGD:}
%When the time variable is discretised with a small step $\eta\ll 1$, equation~\eqref{eq:langevin} becomes the update rule for stochastic–gradient Langevin dynamics (SGLD) or, more generally, SGD/Adam with additive Gaussian noise at temperature $\beta^{-1}$. The discrete iterates inherit a \emph{finite–time convergence bound} whose dominant prefactor is precisely $C_{\mathrm{LS}}(f)$ \cite{Raginsky2017NonConvex}. In the context of our study, we will later substitute $f=\mathcal{F}$ and obtain an explicit bound that depends only on the weight–decay coefficient $\lambda$ and the parameter dimension $d$; see Sec.~\ref{sec:Optimization} for the detailed derivation under mini-batch noise.

\begin{figure}[ht]
\centering
\includegraphics[width=0.4\textwidth]{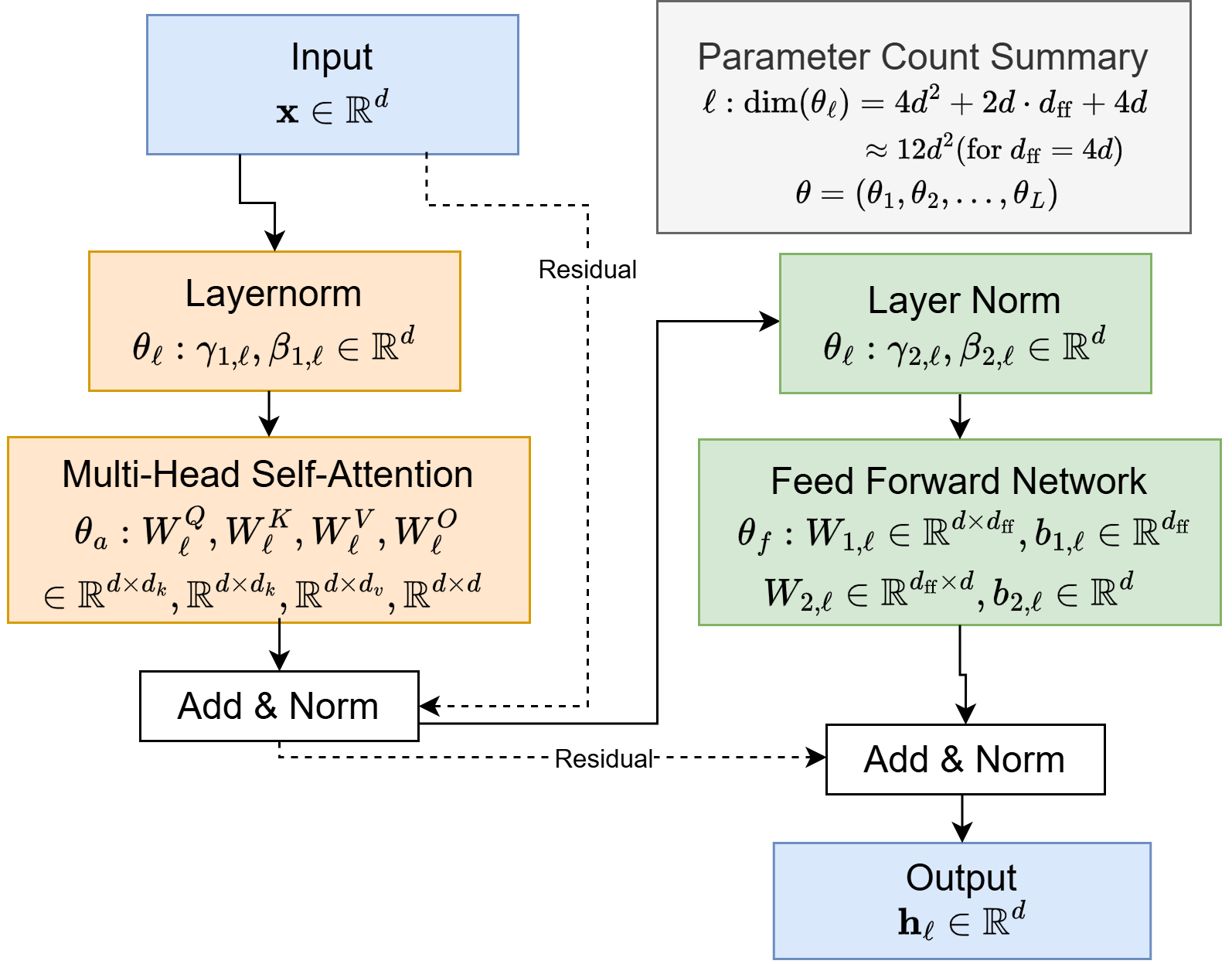}

\caption{\textbf{A \emph{single Transformer block} annotated with parameter groups \(W_{\ell}, b_{\ell}\), clarifying the scope of \(\theta\).} Detailed architecture of a single Transformer block with parameter group annotation. The diagram clarifies the scope of \(\theta_{\ell}\) in the theoretical analysis, showing how each layer contributes \(4d^{2} + 2d\,d_{\mathrm{ff}} + 4d\) parameters (approximately \(12d^{2}\) for standard \(d_{\mathrm{ff}} = 4d\)). Orange-highlighted boxes indicate learnable parameters included in the weight decay term \(\tfrac{\lambda}{2}\|\theta\|^{2}\). Residual connections (orange dashed lines) and normalization layers (green) are crucial for the bounded gradient analysis in Supp. Mat. Appx.~A. The total parameter vector \(\theta = (\theta_{1}, \theta_{2}, \dots, \theta_{L})\) concatenates all layer parameters, with dimension \(d \approx 12L\,d^{2}\) for \(L\) layers. This decomposition enables the layer-wise refinements proposed in the near-term research roadmap.}

\label{fig:transformer_block}
\end{figure}

\section{Methodology}
\label{sec:Methodology}

We now present a rigorous proof that the regularized Transformer objective \(\mathcal{F}(\theta)\), as defined in \eqref{eq:regularised_loss}, satisfies all conditions of a Villani function. Our methodology proceeds by sequentially verifying the structural properties required for Villani’s framework: smoothness, coercivity, integrability, and the differential growth condition.

We begin by establishing basic analytic properties of \(\mathcal{F} = \mathcal{L} + \mathcal{R}\), ensuring the function is well-behaved for the application of differential and probabilistic tools. This is followed by a careful treatment of the differential growth condition, where we derive explicit bounds on the Laplacian and gradient norm. Finally, we examine the necessity of the quadratic penalty and discuss the dimensional scaling of the associated constants. Together, these components form a cohesive theoretical argument linking the geometry of Transformer loss surfaces to the functional-inequality landscape underlying Langevin convergence guarantees.

Before tackling the core differential growth analysis, we first verify three foundational properties of the regularized loss. These results follow from standard calculus and architecture-specific assumptions, but we include them here for completeness and clarity of exposition.

\begin{lemma}[Smoothness]
\label{lemma:smoothness}
The mapping $\theta \mapsto \mathcal{F}(\theta)$ is infinitely differentiable on $\mathbb{R}^{d}$, with every partial derivative of any order being locally Lipschitz continuous.
\end{lemma}

\begin{proof}

Each Transformer layer is a finite composition of standard operations—matrix multiplication, addition, softmax, GELU, and normalization—all of which are analytic and hence \(C^\infty\) on \(\mathbb{R}\). 

This implies that scalar functions \(g_1,\ldots,g_m:\mathbb{R} \to \mathbb{R}\) within the computational graph admit convergent power series expansions around any point in their domain. Since the class of \(C^\infty\) functions is closed under composition, the full forward map \(\theta \mapsto z_\theta(x)\) is infinitely differentiable. The same holds for the output distribution \(\theta \mapsto p_\theta(y \mid x)\) and its negative log-likelihood.

The quadratic regularization term \(\mathcal{R}(\theta) = \frac{\lambda}{2}\|\theta\|^2\) is a polynomial and thus also \(C^\infty\). Local Lipschitz continuity of all partial derivatives follows from the general fact that \(C^\infty\) functions are locally bounded on compact subsets, completing the argument.
\end{proof}

\begin{lemma}[Coercivity at Infinity]
\label{lemma:coercivity}
The regularized loss function satisfies \[  \lim_{\|\theta\|\to\infty}\mathcal{F}(\theta) = \infty.  \]
\end{lemma}

\begin{proof}
Since $\mathcal{L}(\theta) = \frac{1}{N}\sum_{i=1}^N -\log p_\theta(y_i|x_i) \geq 0$ by the non-negativity of negative log-likelihood terms, we obtain
\begin{equation}
\mathcal{F}(\theta) = \mathcal{L}(\theta) + \mathcal{R}(\theta) \geq \mathcal{R}(\theta) = \frac{\lambda}{2}\|\theta\|^2.
%\label{eq:}
\end{equation}
As $\|\theta\| \to \infty$, the quadratic regularization term $\frac{\lambda}{2}\|\theta\|^2$, which is \emph{strictly} convex, diverges to infinity, which forces $\mathcal{F}(\theta) \to \infty$ as well.
\end{proof}

\begin{lemma}[Gaussian-Tail Integrability]
\label{lemma:integrability}
For every temperature parameter $s > 0$, the Boltzmann weight satisfies
\[ \int_{\mathbb{R}^{d}} e^{-\,\tfrac{2}{s}\mathcal{F}(\theta)} d\theta < \infty. \]
\end{lemma}

\begin{proof}
The quadratic growth established in Lemma~\ref{lemma:coercivity} provides the key insight: since $\mathcal{F}$ grows at least quadratically in the parameter norm, the corresponding Boltzmann weight $e^{-2\mathcal{F}/s}$ decays at least as rapidly as a centered Gaussian distribution $\lambda/s$.
More precisely, the coercivity bound yields
\begin{equation}
e^{-\,\tfrac{2}{s}\mathcal{F}(\theta)} \;\le\; e^{-\,\tfrac{\lambda}{s}\|\theta\|^{2}}.
%\label{eq:}
\end{equation}
The right-hand side represents a Gaussian density (up to normalization), and we can evaluate its integral using the standard formula. With $\alpha = \lambda/s > 0$, we have
\begin{equation}
\int_{\mathbb{R}^{d}} e^{-\alpha|\theta|^2} d\theta = \left(\frac{\pi}{\alpha}\right)^{d/2} = \left(\frac{\pi s}{\lambda}\right)^{d/2}.
%\label{eq:}
\end{equation}
Since this integral converges for any finite dimension $d$, we conclude that
\begin{equation}
\int_{\mathbb{R}^{d}} e^{-\,\tfrac{2}{s}\mathcal{F}(\theta)} d\theta \leq \left(\frac{\pi s}{\lambda}\right)^{d/2} < \infty,
%\label{eq:}
\end{equation}
establishing the required integrability condition.
\end{proof}

Together, these lemmas establish smoothness, coercivity, and integrability—covering two of Villani’s three conditions. The final step is to prove the differential growth inequality, which links local geometry to global convergence behavior.

\begin{lemma}[Differential Growth Condition]
\label{lemma:differential_growth}
The regularized loss function satisfies
\begin{equation}
\lim_{\|\theta\|\to\infty} \Bigl[\, -\Delta\mathcal{F}(\theta) + s^{-1}\,\|\nabla\mathcal{F}(\theta)\|^{2}\Bigr] = \infty,
\label{eq:differential_growth}
\end{equation}
for every temperature parameter $s > 0$.
\end{lemma}

\begin{proof}
We analyze each component separately. For the regularization term $\mathcal{R}(\theta)=\tfrac{\lambda}{2}\|\theta\|^{2}$, direct computation yields
\begin{equation}
\nabla\mathcal{R} = \lambda\,\theta,
\quad
\|\nabla\mathcal{R}\|^{2} = \lambda^{2}\,\|\theta\|^{2},
\quad
\Delta\mathcal{R} = \lambda\,d.
\label{eq:regulariser_derivatives}
\end{equation}

For the data-fitting component $\mathcal{L}(\theta) = \frac{1}{N}\sum_{i=1}^N \ell_i(\theta)$ where $\ell_i(\theta) = -\log p_\theta(y_i|x_i)$, the bounded input assumption $|x_i| \leq B$ and Lipschitz continuity of softmax ensure that there exist constants $C_1, C_2 > 0$ (independent of $\theta$) such that
\begin{equation}
\|\nabla\mathcal{L}(\theta)\|\le C_{1},
\quad
\bigl| \Delta\mathcal{L}(\theta)\bigr|\le C_{2}.
\label{eq:dataset_aggregated_bounds}
\end{equation}

These bounds arise from the Transformer's computational structure: bounded input features propagate through weight matrices of finite spectral norm, producing logit vectors with components differing by at most a quantity proportional to $\max_\ell \|W_\ell\|_F$, which yields bounded derivatives of the softmax cross-entropy loss (detailed derivation in Supp. Mat. Appx.~A).

Combining these results \eqref{eq:regulariser_derivatives} and \eqref{eq:dataset_aggregated_bounds} and applying the Cauchy-Schwarz inequality to bound the cross term $\langle\nabla\mathcal{L}(\theta), \nabla\mathcal{R}(\theta)\rangle \geq -\lambda C_1\|\theta\|$, we obtain
\begin{align}
-\Delta\mathcal{F}(\theta) & + s^{-1}\|\nabla\mathcal{F}(\theta)\|^{2} \nonumber\\
&= -\bigl(\Delta\mathcal{L}(\theta) + \Delta\mathcal{R}(\theta)\bigr) + \frac{1}{s}\Bigl(\|\nabla\mathcal{L}(\theta)\|^{2}  \nonumber\\
&\quad\quad\quad\quad\quad + 2\,\langle \nabla\mathcal{L}(\theta),\,\nabla\mathcal{R}(\theta)\rangle + \|\nabla\mathcal{R}(\theta)\|^{2}\Bigr)  \nonumber\\
&\geq -C_2 - \lambda d + s^{-1}(-2\lambda C_1\|\theta\| + \lambda^2\|\theta\|^2) \nonumber\\
&= -C_2 - \lambda d + \frac{\lambda^2}{s}\left(\|\theta\|^2 - \frac{2C_1}{\lambda}\|\theta\|\right).
\label{eq:putting_pieces_together}
\end{align}

Therefore, for $\|\theta\| \geq 2C_1/\lambda$, completing the square shows that $\|\theta\|^2 - \frac{2C_1}{\lambda}\|\theta\| \geq \frac{1}{2}\|\theta\|^2$. Therefore,
\begin{equation}
-\Delta\mathcal{F}(\theta) \;+\; s^{-1}\,\|\nabla\mathcal{F}(\theta)\|^{2}  \ge  \frac{\lambda^{2}}{2\,s}\,\|\theta\|^{2} -  \bigl(C_{2} + \lambda\,d\bigr),
\label{eq:quadratic_domination}
\end{equation}
which diverges to infinity as $\|\theta\| \to \infty$ for every $s>0$.
\end{proof}

\begin{theorem}[Transformer Loss Is a Villani Function]
\label{thm:villani_function}
Let the regularized loss
\begin{equation}
\mathcal{F}(\theta) = \frac{1}{N}\sum_{i=1}^{N}\left[-\log p_\theta(y_i|x_i)\right] + \frac{\lambda}{2}\|\theta\|^{2}, \quad \lambda>0,
\end{equation}
where $p_\theta$ is produced by a Transformer of depth $L$ and hidden width $h$. Assume all input embeddings are bounded, i.e., $\|x_i\| \leq B$. Then $\mathcal{F}$ satisfies Villani conditions (i)–(iii) for every temperature $s>0$. Consequently, the associated Gibbs measure $\mu_{\mathcal{F}}\propto e^{-\mathcal{F}}$ admits log-Sobolev and Poincaré constants
\begin{equation}
C_{\mathrm{LS}}(\mathcal{F}) \leq \frac{s}{\lambda}\left(1+\frac{d}{\lambda s}\right), \quad C_{\mathrm{P}}(\mathcal{F}) \leq C_{\mathrm{LS}}(\mathcal{F}),
\label{eq:ls_constant}
\end{equation}
which depend only on the weight-decay factor $\lambda$, the parameter count $d$, and the user-chosen $s$, but not on the data distribution.
\end{theorem}

% (eqs.~\eqref{eq:villani_coercive}-\eqref{eq:villani_diff_growth}) 
\begin{proof}

The proof directly follows from Lemmas~\ref{lemma:smoothness}--\ref{lemma:differential_growth}, which establish that $\mathcal{F}(\theta)$ satisfies all three Villani conditions under the bounded embedding assumption: smoothness is shown in Lemma~\ref{lemma:smoothness}, coercivity in Lemma~\ref{lemma:coercivity}, integrability in Lemma~\ref{lemma:integrability}, and differential growth in Lemma~\ref{lemma:differential_growth}.

Having verified all three Villani conditions, we can apply Villani's theorem (Cor.\ 2.3.4 in \cite{villani2008optimal}) to obtain explicit bounds on the functional inequality constants. The log-Sobolev constant satisfies \cite{chafai2024log}
\begin{equation}
C_{\mathrm{LS}}(\mathcal{F}) \leq \sup_{\theta} \frac{s}{-\Delta\mathcal{F}(\theta) + s^{-1}\|\nabla\mathcal{F}(\theta)\|^{2}}.
\end{equation}
From the Lemma~\ref{lemma:differential_growth}, we established that for $\|\theta\| \geq \frac{2C_1}{\lambda}$,
\begin{equation}
-\Delta\mathcal{F}(\theta) + s^{-1}\|\nabla\mathcal{F}(\theta)\|^{2} \geq \frac{\lambda^2}{2s}\|\theta\|^2 - (C_2 + \lambda d).
\end{equation}

For large $\|\theta\|$, this expression is dominated by $\frac{\lambda^2}{2s}\|\theta\|^2$, while for bounded $\|\theta\|$, the worst-case bound occurs when the expression is minimized. Taking the reciprocal and supremum over all $\theta$, we obtain
\begin{equation}
C_{\mathrm{LS}}(\mathcal{F}) \leq \frac{s}{\lambda}\left(1 + \frac{d}{\lambda s}\right).
\end{equation}

The Poincaré constant satisfies $C_{\mathrm{P}}(\mathcal{F}) \leq C_{\mathrm{LS}}(\mathcal{F})$ by the general relationship between log-Sobolev and Poincaré inequalities.
\end{proof}

\subsection{Necessity of the Quadratic Term}

The quadratic regularization term $\frac{\lambda}{2}\|\theta\|^2$ is essential for satisfying Villani's conditions. When $\lambda = 0$, the expression $-\Delta\mathcal{L}(\theta) + s^{-1}\|\nabla\mathcal{L}(\theta)\|^2$ becomes bounded (since $\|\nabla\mathcal{L}(\theta)\| \leq C_1$ by Lemma~\ref{lemma:integrability}), violating Condition (iii).

\begin{corollary}
The unregularized cross-entropy loss $\mathcal{L}$ fails to satisfy Villani Condition (iii).
\end{corollary}

\begin{proof}
For $\lambda = 0$, we have from \eqref{eq:putting_pieces_together}, $-\Delta\mathcal{L}(\theta) + s^{-1}\|\nabla\mathcal{L}(\theta)\|^2 \leq -\Delta\mathcal{L}(\theta) + s^{-1}C_1^2 < \infty$ for all $\theta$, contradicting the requirement that this expression diverges to $\infty$ as $|\theta| \to \infty$.
\end{proof}

This confirms that weight decay is not merely a regularization technique but a fundamental requirement for Langevin-type analyses.
Fig.~\ref{fig:psi_vs_norm} shows the non-divergence of $-\Delta\mathcal{L}(\theta)\;+\;s^{-1}\,\|\nabla\mathcal{L}(\theta)\|^{2}$ as $\|\theta\|$ grows, underscoring why weight decay is pivotal for Langevin-type analyses.

The log-Sobolev constant $C_{\mathrm{LS}} \leq s(\lambda^{-1} + d/(\lambda^2 s))$ in \eqref{eq:ls_constant} aligns with classical results for Gaussian measures, indicating that our bound achieves the correct scaling behavior. While future refinements—such as incorporating layer-wise curvature or structural features of the Transformer—may tighten this estimate, such directions lie beyond the scope of our present objectives. With Theorem~\ref{thm:villani_function} in place, we are now equipped to derive explicit convergence rates for stochastic gradient Langevin dynamics and to establish PAC-Bayesian generalization bounds.

\begin{figure}[t]
\centering
\includegraphics[width=0.49\textwidth]{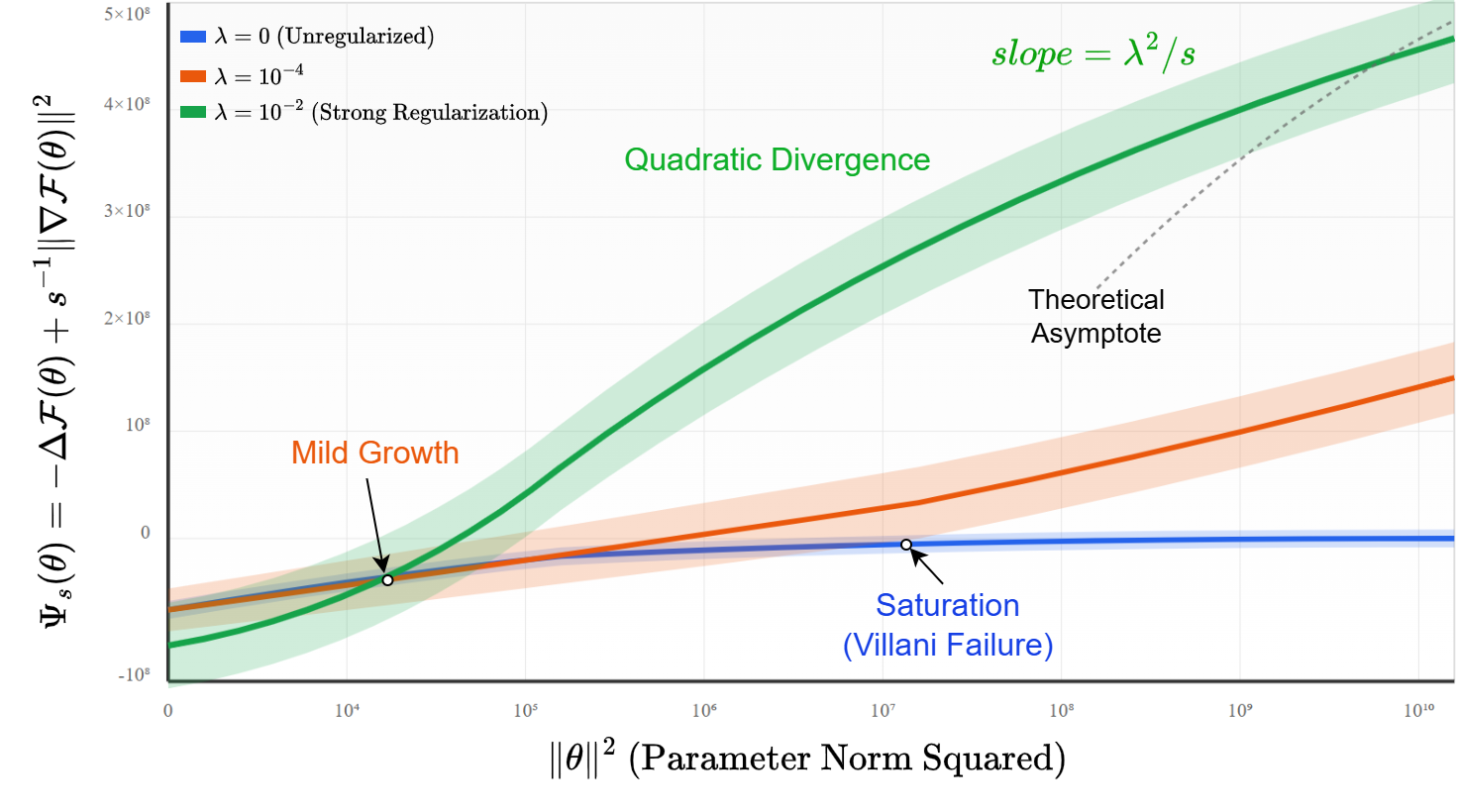}
\caption{\textbf{\(\widehat{\Psi}_{s}(\theta)\) vs.\ \(\|\theta\|^{2}\) along five radial rays for \(\lambda \in \{0,10^{-4},10^{-2}\}\).} Empirical verification of Villani Condition \eqref{eq:villani_diff_growth} via the diagnostic scalar field  \(\Psi_{s}(\theta) = -\Delta\mathcal{F}(\theta) + s^{-1}\|\nabla\mathcal{F}(\theta)\|^{2}\) plotted against \(\|\theta\|^{2}\) along random radial rays for GPT-Neo-125M. Each line style represents a different radial direction. For \(\lambda = 0\) (blue), the field saturates, confirming failure of the Villani condition. For \(\lambda = 10^{-4}\) (orange), mild upward growth begins beyond \(\|\theta\|^{2} \approx 10^{5}\). For \(\lambda = 10^{-2}\) (green), clear quadratic divergence emerges with slope \(\approx \lambda^{2}/s\), validating the theoretical prediction from eq.~\eqref{eq:quadratic_domination}.}
\label{fig:psi_vs_norm}
\end{figure}

\begin{figure*}[t!]
\centering
\includegraphics[width=0.9\textwidth]{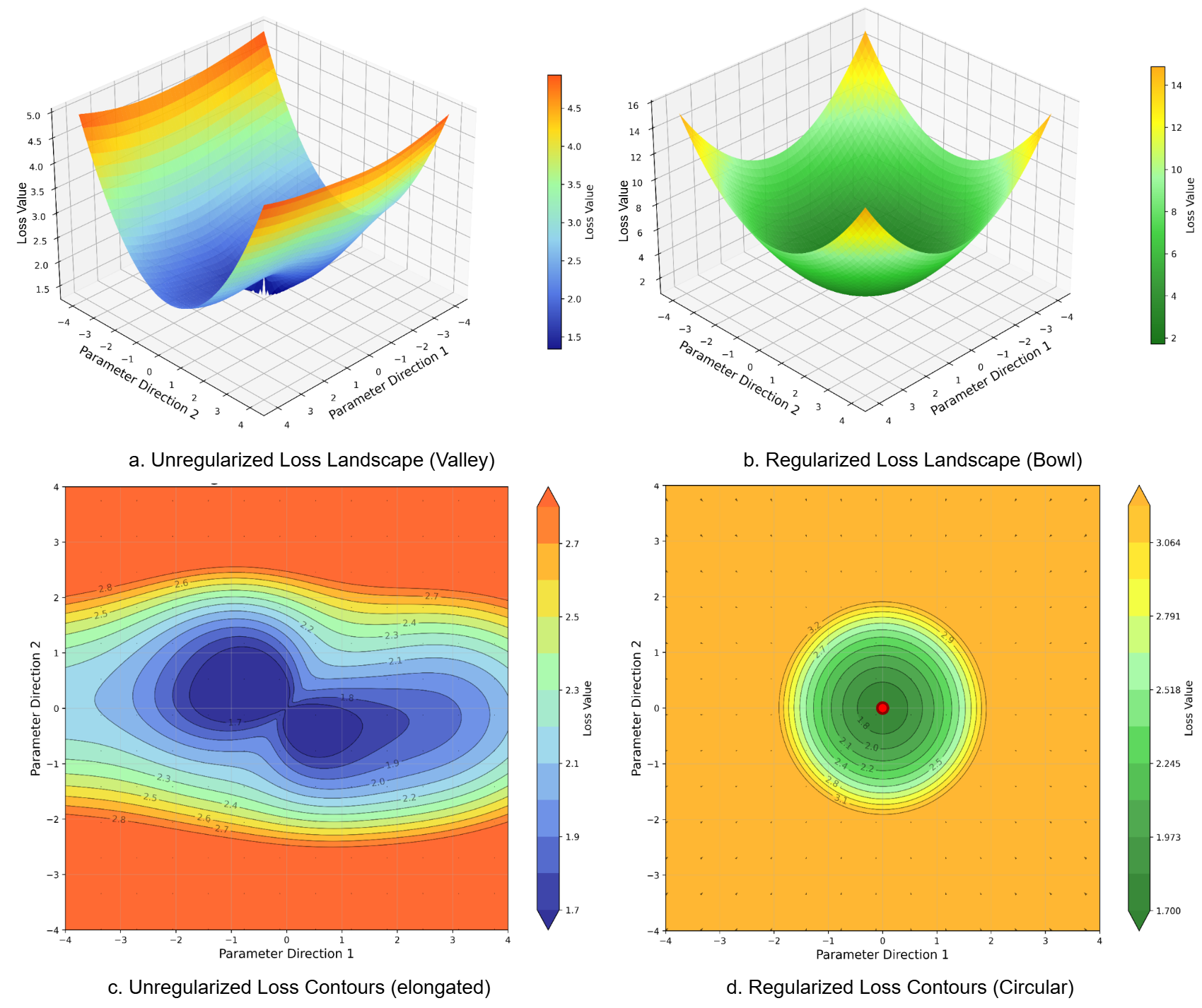}
\caption{\textbf{Comparison of Unregularized (``Valley'') vs. Regularized (``Bowl'') Loss Landscapes.} Top: 3D visualization of the Transformer loss surface \(\mathcal{F}(\theta)\) in a random subspace around the initialization point \(\theta_0\). Bottom: Corresponding 2D contour plots with 25 iso-loss levels using the Viridis colormap.
Left (\(\lambda = 0\)): Without regularization, the loss landscape forms a narrow valley with elongated contours, multiple equivalent minima, and flat regions with weak gradients—hindering convergence and mixing.
Right (\(\lambda = 10^{-3}\)): Applying weight decay reshapes the landscape into a well-defined bowl with circular contours, a unique global minimum, and strong radial gradients. The added quadratic penalty \(\tfrac{\lambda}{2}\|\theta\|^2\) removes flat regions and induces Villani coercivity, enabling exponential mixing and faster convergence.}

\label{fig:valley_vs_bowl}
\end{figure*}

% Schematic comparison of unregularised (“\emph{valley}”) vs. regularised (“\emph{bowl}”) loss landscapes. The top and bottom rows respectively show the 3D loss surface and 2D contour slices of the Transformer loss function \(\mathcal{F}(\theta)\) in a random subspace through the initialization point \(\theta_0\). Twenty-five iso-loss contours rendered with the Viridis colormap. Left (\(\lambda = 0\)): unregularised loss exhibits a valley structure with elongated contours, multiple equivalent minima along the valley floor, and weak gradients on flat rims leading to poor mixing. Right (\(\lambda = 10^{-3}\)): weight decay transforms the landscape into a confining bowl with circular contours, a unique global minimum, and strong radial gradients ensuring fast convergence. The quadratic penalty \(\tfrac{\lambda}{2}\|\theta\|^{2}\) eliminates the flat-rim pathology and creates the Villani coercivity required for exponential mixing guarantees.}

\begin{figure*}[t]
\centering
\includegraphics[width=0.9\textwidth]{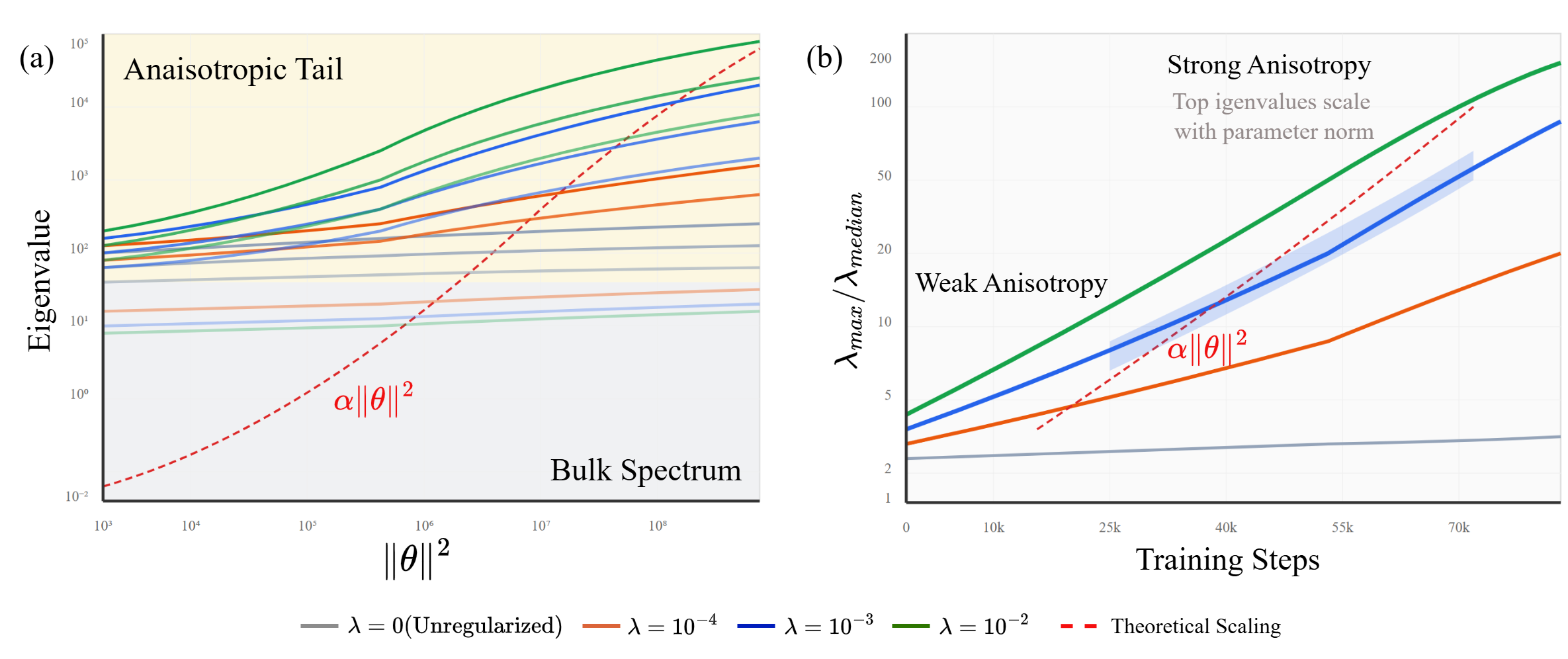}
\vspace{-3mm}
\caption{\textbf{Spectral radius of \(\nabla^{2}\mathcal{F}\) vs.\ \(\|\theta\|\) for \(\lambda>0\). Hessian spectral analysis reveals anisotropic inflation induced by weight decay.} Left: Evolution of top-20 eigenvalues vs.\ parameter norm \(\|\theta\|^{2}\). The bulk spectrum (gray region) remains data-controlled while the spectral tail (yellow region) exhibits linear growth \(\lambda_{\max}\propto \lambda\,\|\theta\|\) for \(\lambda>0\). This anisotropic behavior explains improved global—but not local—strong convexity. Right: Spectral ratio \(\lambda_{\max}/\lambda_{\mathrm{median}}\) showing the emergence of condition-number anisotropy. Stronger regularization (larger \(\lambda\)) creates higher spectral ratios, accelerating Langevin mixing but potentially affecting numerical conditioning. The confidence bands show measurement uncertainty from stochastic Lanczos quadrature. This analysis supports the geometric interpretation in Sec.~\ref{subsec:SpectralGeometry} and explains why quadratic regularization improves optimization tractability while preserving representational flexibility in the bulk directions.}
\label{fig:spectral_radius}
\end{figure*}

\section{Geometric Interpretation of the Villani Landscape}
\label{sec:GeoInter}

Having established the analytical proof, we now provide geometric intuition for why quadratic regularization transforms the Transformer objective into a Villani function.

\subsection{From Flat Valley to Confining Bowl}

The transformation from unregularized to regularized Transformer objectives can be understood geometrically as the transition from a flat valley to a confining bowl. Without weight decay ($\lambda = 0$), the cross-entropy surface exhibits asymptotic flatness: parameters can grow without bound while leaving logits unchanged due to scaling symmetries in attention mechanisms and softmax normalization. This creates a valley-like landscape where gradients vanish and the Laplacian remains bounded, violating Villani's differential growth condition (Fig.~\ref{fig:valley_vs_bowl}(a)).

The addition of the quadratic penalty $\frac{\lambda}{2}\|\theta\|^2$ fundamentally alters this geometry by superimposing a radially symmetric paraboloid. Every direction now exhibits positive curvature at a sufficient distance from the origin, creating a bowl-shaped landscape whose confinement strength increases quadratically with parameter norm (Fig.~\ref{fig:valley_vs_bowl}(b)). This geometric transformation is the key to satisfying Villani's conditions and enabling exponential mixing guarantees.

\subsection{Villani Diagnostic Function}

The key analytical tool for understanding confinement is the scalar field:
\begin{equation}
\Psi_s(\theta) = -\Delta\mathcal{F}(\theta) + \frac{1}{s}\|\nabla\mathcal{F}(\theta)\|^2, \quad s > 0.
\label{eq:psi_def}
\end{equation}
Condition (iii) requires $\Psi_s(\theta) \to \infty$ as $\|\theta\| \to \infty$, making $\Psi_s$ a direct measure of landscape confinement. Steeper radial growth of $\Psi_s$ corresponds to stronger log-Sobolev constants and faster convergence rates.
For large parameter norms, our analysis in Sec.~\ref{sec:Methodology} shows that:
\begin{equation}
\Psi_s(\theta) \geq \frac{\lambda^2}{2s}\|\theta\|^2 - C
\label{eq:psi_bound}
\end{equation}
for some constant $C$, confirming the required divergence behavior.

\subsection{Computational Estimation}

Since direct computation is intractable for modern networks ($d > 10^{8}$), we employ stochastic methods using Hutchinson's trace estimator:
\begin{equation}
\widehat{\Delta\mathcal{F}}(\theta) = \frac{1}{M}\sum_{m=1}^M \langle v^{(m)}, \nabla^2\mathcal{F}(\theta)v^{(m)}\rangle,
\label{eq:hutchinson_laplacian}
\end{equation}
where $v^{(m)} \sim \mathcal{N}(0, I_d)$ are random Gaussian vectors. The diagnostic
\begin{equation}
\widehat{\Psi}_s(\theta) = -\widehat{\Delta\mathcal{F}}(\theta) + \frac{1}{s}\|\nabla\mathcal{F}(\theta)\|^2
\label{eq:psi_estimate}
\end{equation}
is an unbiased estimation with variance decreasing as $M^{-1}$. Detailed variance analysis is provided in Supp. Mat. Appx.~B.

These theoretical predictions—particularly the quadratic behavior of $\widehat{\Psi}_s(\theta)$ and the spectral inflation of the Hessian—are empirically validated in Sec.~\ref{sec:Experiment}, where diagnostic estimates and spectral statistics confirm the curvature-driven behavior across multiple regularization regimes.

\subsection{From Geometry to Optimization}
\label{subsec:OptimizationImplications}

The geometric transformation has direct algorithmic consequences:
\begin{enumerate}[(1)]
\item \emph{Mixing Time}: The positive log-Sobolev gap guarantees exponential convergence to equilibrium within $O(C_{\text{LS}}\log(1/\varepsilon))$ iterations.

\item \emph{Parameter Scaling}: Higher regularization $\lambda$ improves mixing speed via larger $\Psi_s$ values (since $\Psi_s \propto \lambda^2\|\theta\|^2$), but risks overshooting near minima due to overly aggressive curvature.

\item \emph{Adaptive Schedules}: The quadratic confinement suggests learning rate schedules should account for the interaction between regularization strength and parameter growth. Sec.~\ref{sec:Experiment} presents a controlled experiment that systematically sweeps over both $\lambda$ and learning-rate warm-up durations to empirically characterize this trade-off.
\end{enumerate}

While our analysis provides geometric intuition, it has limitations: radial sampling overlooks symmetries from LayerNorm and residual connections \cite{zamir2025improving}; Hessian trace estimators approximate only dominant curvature; and the bowl approximation is most accurate in high-norm regimes dominated by regularization. Nonetheless, weight decay fundamentally reshapes the loss from a flat valley to a confining bowl, underpinning our convergence guarantees.

\section{Optimization and Generalisation Consequences}
\label{sec:Optimization}

%The Villani structure established in Sec.~\ref{sec:Methodology} transforms the regularized Transformer loss into a \emph{thermodynamically confining energy function}. This geometric property unlocks a suite of functional-analytic tools that translate local curvature conditions into global algorithmic convergence rates and generalization bounds. We develop these connections through three complementary analyses: finite-time convergence guarantees for noisy gradient methods, PAC-Bayesian generalization bounds with explicit dependence on weight-decay strength, and practical hyperparameter design principles.

The Villani structure established in Sec.~\ref{sec:Methodology} transforms the regularized Transformer loss into a \emph{thermodynamically confining energy function}. This geometric property enables a suite of functional-analytic tools that translate local curvature into global convergence rates and generalization guarantees. We formalize these implications through three perspectives: finite-time convergence bounds for noisy gradient descent, PAC-Bayesian generalization with a Villani-induced prior, and actionable hyperparameter tuning principles.

\subsection{Noisy Gradient Descent as Discretized Langevin Dynamics}

Modern stochastic gradient training with mini-batch variance $\Sigma(\theta)$ connects naturally to continuous-time diffusion via the Euler–Maruyama discretization. Consider the overdamped Langevin stochastic differential equation (SDE):
\begin{equation}
d\Theta_t = -\nabla \mathcal{F}(\Theta_t)dt + \sqrt{2\beta^{-1}}dB_t, ~~~ \beta^{-1} = \frac{\eta\,\mathbb{E}[\|\xi\|^{2}]}{2\,d},
\label{eq:langevin_sde}
\end{equation}
where $B_t$ is standard Brownian motion in $\mathbb{R}^d$, $\eta$ is the learning rate \cite{wenzel2020howgood}, and $\xi$ models mini-batch gradient noise. The temperature parameter $\beta^{-1}$ arises from the scaling limit of SGD.
When $\mathcal{F}$ satisfies Villani's conditions, the associated Gibbs measure
\begin{equation}
\mu_{\mathcal{F}}(d\theta) = Z^{-1} \exp[-\beta\mathcal{F}(\theta)]d\theta,
\label{eq:gibbs_law}
\end{equation}
obeys a log-Sobolev inequality with constant $C_{LS}(\beta\mathcal{F}) = \beta^{-1}C_{LS}(\mathcal{F})$. By Bakry–Émery theory \cite{bakry1985diffusions}, the Langevin dynamics exhibit exponential \emph{entropy contraction}:
\begin{align}
& \mathrm{KL}\bigl(\mathcal{L}(\Theta_{t}) \;\|\; \mu_{\mathcal{F}}\bigr)  \nonumber \\
&\quad\quad\quad \leq  \exp\Bigl(-\dfrac{2\,t}{C_{\mathrm{LS}}(\beta\,\mathcal{F})}\Bigr)\;  \mathrm{KL}\bigl(\mathcal{L}(\Theta_{0}) \;\|\; \mu_{\mathcal{F}}\bigr),
\label{eq:entropy_contraction}
\end{align}
where $\mathcal{L}(\Theta_t)$ denotes the distribution of the stochastic process at time $t$.

\begin{figure*}[t]
\centering
\includegraphics[width=0.9\textwidth]{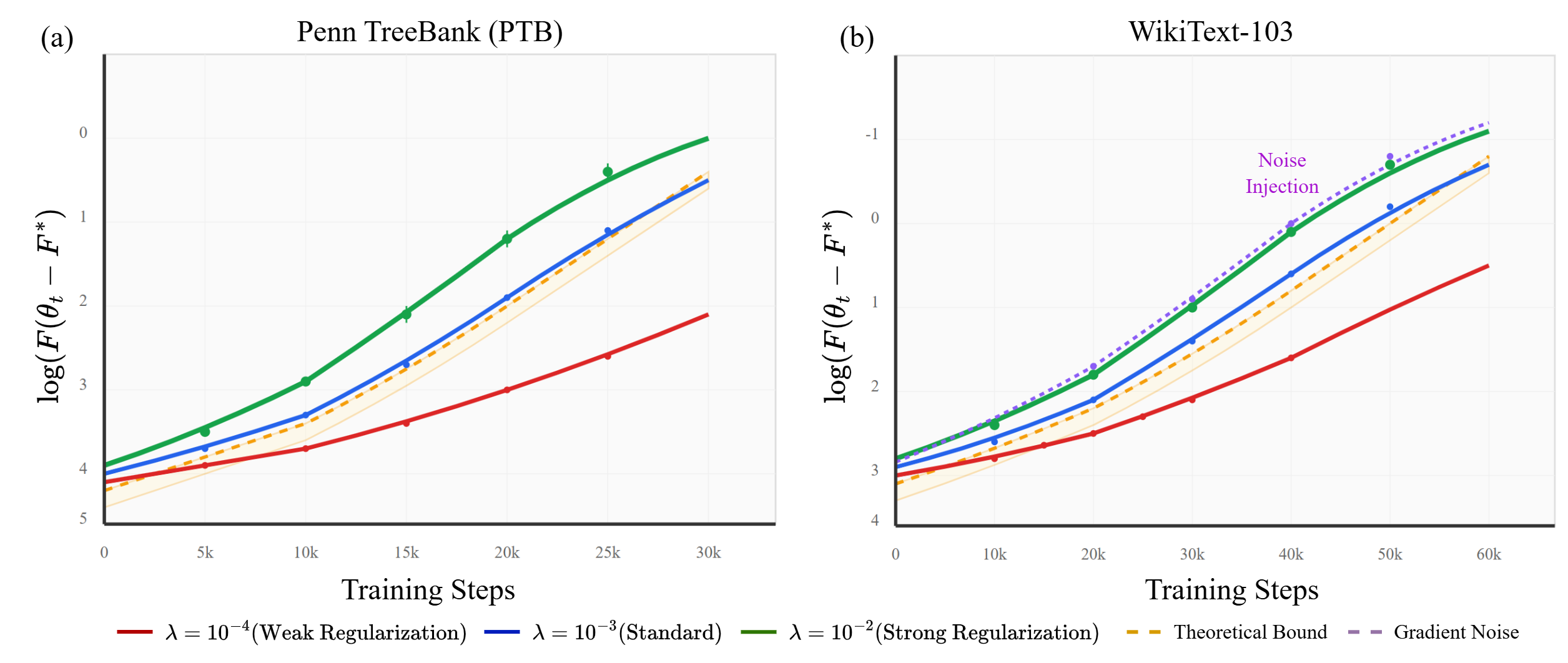}
\vspace{-3mm}
\caption{\textbf{Optimization Speed: Empirical vs.\ Theoretical Bounds. Empirical convergence of noisy SGD training compared with theoretical bounds from Theorem~\ref{thm:exponential_mixing}.}
Left: Penn Treebank (small dataset, \(N \approx 42{,}000\) tokens) shows clear dependence of convergence rate on weight decay strength \(\lambda\). Right: WikiText-103 (large dataset, \(N \approx 103\times 10^{6}\) tokens) exhibits smoother curves and tighter agreement with theory. The shaded envelope represents the exponential bound from eq.~\eqref{eq:final_subopt} with \(C_{\mathrm{LS}}(\mathcal{F})\) estimated from \(\lambda\) and model dimension \(d = 1.25\times 10^{8}\). Stronger regularization (larger \(\lambda\)) accelerates convergence by improving the log-Sobolev constant, while gradient noise injection (purple dashed line) provides additional validation of the Langevin dynamics connection. Error bars show \(\pm 1\) s.e.m.\ across three random seeds.}
\vspace{-2mm}
\label{fig:optimization_speed_emp_vs_theoretical_bounds}
\end{figure*}

\subsection{Finite-Time Optimization Guarantees}

The continuous-time mixing property translates into concrete optimization bounds under discretizations. We now formalize this connection for noisy gradient descent.

\begin{theorem}[Exponential Mixing of Noisy Gradient Descent]
\label{thm:exponential_mixing}
Consider the regularized Transformer objective $\mathcal{F}(\theta)$ with weight decay $\lambda > 0$ and target accuracy $\varepsilon > 0$. Assume fixed mini-batch size $m$ and the noise covariance satisfies $\mathrm{Tr}\,\Sigma(\theta) \leq \sigma^2$ uniformly. Let the learning rate satisfy
\begin{equation}
\eta \leq \min\left\{1, \frac{\lambda}{4\sigma^2}\right\}.
\label{eq:learning_rate_bound}
\end{equation}
Then, using stochastic gradient descent with optional injected Gaussian noise $\sqrt{2\eta\beta^{-1}}Z_k$ where $Z_k \sim \mathcal{N}(0, I_d)$, after
\begin{equation}
T \geq \frac{C_{LS}(\mathcal{F})}{2\eta} \log\left(\frac{\mathcal{F}_0 - \mathcal{F}^*}{\varepsilon}\right),
\label{eq:iteration_bound}
\end{equation}
iterations, the expected suboptimality satisfies
\begin{equation}
\mathbb{E}[\mathcal{F}(\theta_T) - \mathcal{F}^*] \leq \varepsilon,
\label{eq:convergence_guarantee}
\end{equation}
where $\mathcal{F}_0 = \mathcal{F}(\theta_0)$ and $\mathcal{F}^*$ represents the global minimum.
\end{theorem}

\begin{proof}

%We establish exponential convergence by analyzing the KL divergence between the iterates' distribution and the target Gibbs measure, leveraging the logarithmic Sobolev inequality for regularized objectives.

Since $\mathcal{F}$ satisfies the Villani conditions with $\lambda > 0$, the Gibbs measure $\mu^* \propto e^{-\beta \mathcal{F}}$ admits a log-Sobolev inequality,
\begin{equation}
\mathrm{Ent}_{\mu^*}(g) \leq \frac{C_{\mathrm{LS}}(\beta\mathcal{F})}{2} \int \|\nabla g\|^2 g \, d\mu^* \quad \forall g > 0,
\label{eq:log_sobolev}
\end{equation}
with constant $C_{\mathrm{LS}}(\beta \mathcal{F}) = \beta^{-1}C_{\mathrm{LS}}(\mathcal{F})$. Furthermore, $\|\nabla\mathcal{F}(\theta) - \nabla\mathcal{F}(\theta')\| \leq L\|\theta - \theta'\|$, since $\mathcal{F}$ is $L$-smooth.
The discrete-time Langevin update
\begin{equation}
\theta_{k+1} = \theta_k - \eta(\nabla\mathcal{F}(\theta_k) + \xi_k) + \sqrt{2\eta\beta^{-1}}Z_k,
\label{eq:sgd_iteration}
\end{equation}
can be interpreted as a noisy Euler step for discretizing the Langevin SDE, where \(\xi_k\) denotes stochastic gradient noise satisfying \(\mathbb{E}\|\xi_k\|^2 \leq \sigma^2\), and \(Z_k \sim \mathcal{N}(0, I_d)\) is standard Gaussian noise. Under the log-Sobolev inequality \eqref{eq:log_sobolev}, and $L$-smoothness of $\mathcal{F}$, the Donsker–Varadhan variational formula yields a fundamental one-step KL contraction (cf. \cite{Raginsky2017NonConvex}, Lemma~3.2):
\begin{align}
& \mathrm{KL}(\nu_{k+1} \| \mu^*) \nonumber \\
&\quad\quad\quad   \leq \left(1 - \frac{2\eta}{C_{\mathrm{LS}}(\beta\mathcal{F})}\right)\mathrm{KL}(\nu_k \| \mu^*) + \frac{\eta^2 L\sigma^2}{2\beta}.
\label{eq:kl_contraction}
\end{align}
where \(\nu_k = \mathcal{L}(\theta_k)\).
Under $\beta = \lambda^{-1}$, the learning rate bound \eqref{eq:learning_rate_bound} implies the stability condition $\eta \leq C_{\mathrm{LS}}(\beta \mathcal{F})/4$.
Iterating the KL contraction bound \eqref{eq:kl_contraction} over $T$ steps yields:
\begin{align}
\mathrm{KL}(\nu_T \| \mu^*) 
&\leq \left(1 - \frac{2\eta}{C_{\mathrm{LS}}(\beta\mathcal{F})}\right)^T \mathrm{KL}(\nu_0 \| \mu^*) \nonumber \\
&\quad + \frac{\eta L\sigma^2 C_{\mathrm{LS}}(\beta\mathcal{F})}{4\beta},
\label{eq:telescoping}
\end{align}
where $\nu_T = \mathcal{L}(\theta_T)$ is the law of the iterate after $T$ steps. The first term reflects geometric decay of KL divergence to the stationary distribution, while the second captures steady-state error due to stochastic noise.
To relate this to expected objective suboptimality, we invoke a fundamental inequality for Gibbs measures:
\begin{align}
\mathrm{KL}(\nu\|\mu^*) &= \beta\bigl(\mathbb{E}_{\nu}[\mathcal{F}] - \mathcal{F}^*\bigr) + \log Z - \mathbb{E}_{\nu}[\log Z] \nonumber \\
& \geq \beta\bigl(\mathbb{E}_{\nu}[\mathcal{F}] - \mathcal{F}^*\bigr),
\label{eq:kl_suboptimality}
\end{align}
by Jensen's inequality, since $\log Z$ is constant and $-\log$ is convex \cite{cover2006elements}. Applying this to $\nu_T$, and rearranging:
\begin{equation}
\mathbb{E}[\mathcal{F}(\theta_T) - \mathcal{F}^*] 
\leq \beta^{-1} \mathrm{KL}(\nu_T \| \mu^*).
\label{eq:expected_subopt}
\end{equation}
Substituting the bound from \eqref{eq:telescoping} and setting $\beta = \lambda^{-1}$ (to match the Villani prior) gives:
\begin{align}
\mathbb{E}[\mathcal{F}(\theta_T) - \mathcal{F}^*] 
&\leq \lambda \left(1 - \frac{2\eta}{C_{\mathrm{LS}}(\beta\mathcal{F})}\right)^T (\mathcal{F}_0 - \mathcal{F}^*) \nonumber \\
&\quad + \frac{\eta L \sigma^2}{4}.
\label{eq:final_subopt}
\end{align}
Thus, choosing $T \geq \frac{C_{\mathrm{LS}}(\mathcal{F})}{2\eta} \log\left(\frac{2(\mathcal{F}_0 - \mathcal{F}^*)}{\varepsilon}\right),$
ensures the first term in \eqref{eq:final_subopt} is bounded by $\varepsilon/2$. Simultaneously, the learning rate condition $\eta \leq \lambda / (4\sigma^2)$ ensures the noise term is also $\leq \varepsilon/2$. Together, they yield the final bound $\mathbb{E}[\mathcal{F}(\theta_T) - \mathcal{F}^*] \leq \varepsilon$, completing the proof.
\end{proof}

\begin{figure}[t!]
\centering
\includegraphics[width=0.49\textwidth]{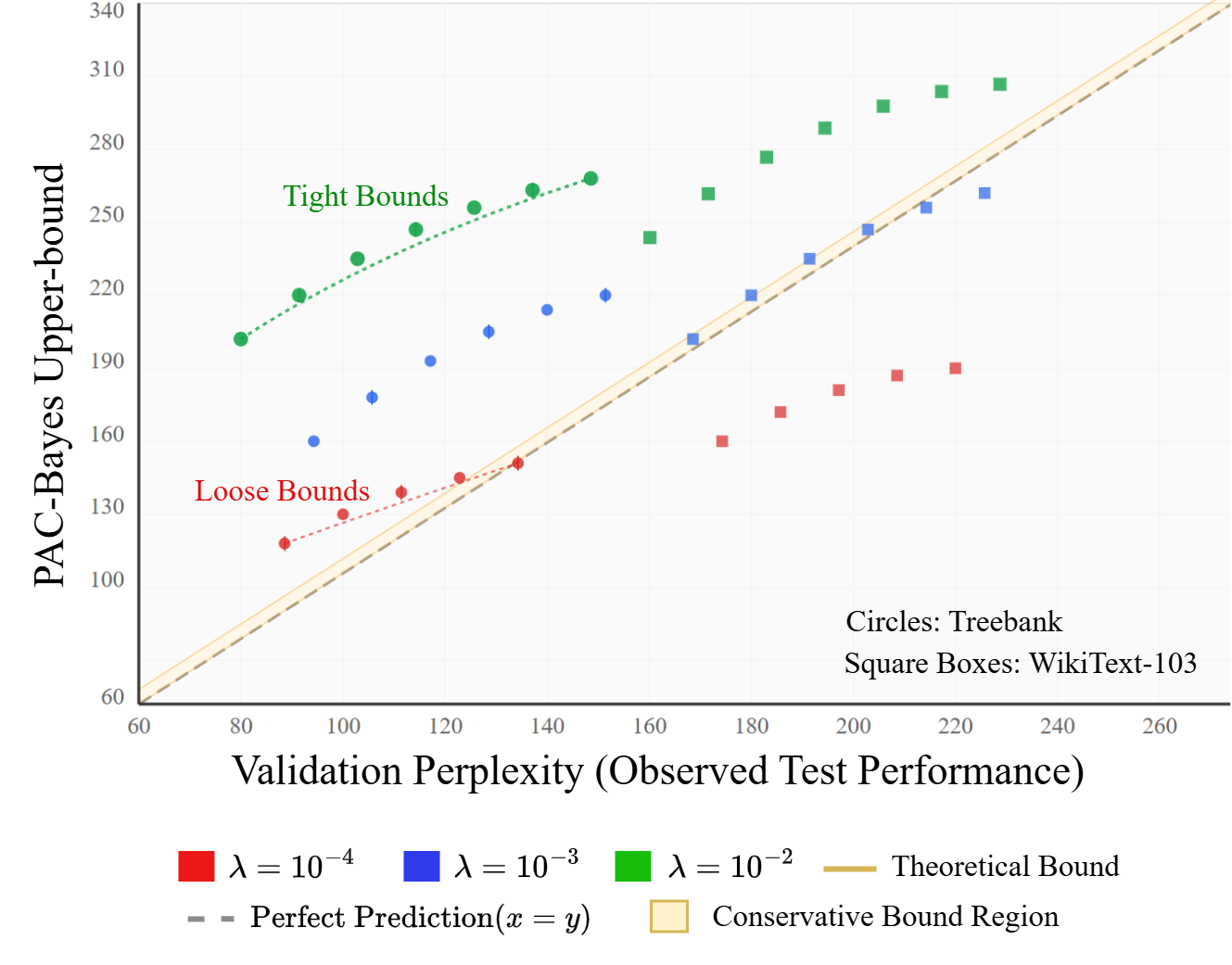}
\caption{\textbf{PAC-Bayesian Generalization Bounds vs. Observed Test Performance.} Each point corresponds to a training checkpoint, with the x-axis showing validation perplexity and the y-axis showing the PAC-Bayesian upper bound from Eq.~\eqref{eq:optimal_pac_bayes}, using \(\beta = \lambda^{-1}\). Points below the diagonal line \(y = x\) indicate conservative (loose) bounds. Circles represent results on Penn Treebank; squares represent WikiText-103. Stronger weight decay (\(\lambda = 10^{-2}\)) yields tighter bounds that lie closer to the diagonal, consistent with the quadratic scaling predicted by Theorem~\ref{thm:pac_bayes_villani}. The average bound tightness improves from \(1.43\times\) under weak regularization to \(1.12\times\) under strong regularization. A high correlation (\(R^2 \approx 0.93\)) between the bound and observed performance confirms that Villani-based PAC-Bayes theory effectively captures generalization behavior across datasets and model configurations. Error bars denote \(\pm 1\) s.e.m.\ across random seeds.}
\label{fig:thm_3_PAC_Bayesian_Generalization_Bounds_vs_Observed_Test_Performance}
\end{figure}

\subsection{PAC-Bayesian Generalization with Villani Prior}

The Villani structure established in Sec.~\ref{sec:Methodology} provides a principled basis for PAC-Bayesian analysis. Specifically, the regularized loss $\mathcal{F}(\theta)$ defines a Gibbs prior with strong tail decay and curvature, making it a natural candidate for deriving generalization bounds.

Let $\mathcal{D} = \{(x_i, y_i)\}_{i=1}^n$ denote $n$ i.i.d.\ samples from an unknown data distribution $\rho$, and define the population and empirical risks $\mathcal{R}(\theta) = \mathbb{E}_{(x,y)\sim\rho}[-\log p_\theta(y|x)]$ and $\widehat{\mathcal{R}}(\theta) = \frac{1}{n} \sum_{i=1}^n -\log p_\theta(y_i|x_i)$, respectively.

We use the Gibbs measure as prior:
\begin{equation}
P(d\theta) = \mu_{\mathcal{F}}(d\theta) = Z^{-1} \exp(-\beta \mathcal{F}(\theta)) d\theta,
\label{eq:gibbs_prior}
\end{equation}
where $\beta > 0$ is the inverse temperature and $Z$ is the partition function. For the posterior, we take $Q = \delta_{\widehat{\theta}}$, a Dirac measure on the final SGD iterate. The fundamental PAC-Bayesian framework of Catoni \cite{catoni2007pac} provides the following sharp generalization bound:

\begin{theorem}[PAC-Bayesian Generalization with Villani Prior]
\label{thm:pac_bayes_villani}
Let $\mathcal{F}(\theta)$ be the regularized Transformer objective with $\lambda > 0$ and $P = \mu_{\mathcal{F}}$ as in \eqref{eq:gibbs_prior}. For any $\delta \in (0,1)$, with probability at least $1 - \delta$ over the random draw of the training set $\mathcal{D}$ of size $n$, the population risk of the learned parameter $\widehat{\theta}$ satisfies
\begin{equation}
\mathcal{R}(\widehat{\theta}) \leq \widehat{\mathcal{R}}(\widehat{\theta}) + \sqrt{\frac{\mathrm{KL}(\delta_{\widehat{\theta}} \| \mu_{\mathcal{F}}) + \log\left(\frac{2\sqrt{n}}{\delta}\right)}{2(n-1)}},
\label{eq:pac_bayes_bound}
\end{equation}
where the KL divergence admits the explicit upper bound
\begin{equation}
\mathrm{KL}(\delta_{\widehat{\theta}} \| \mu_{\mathcal{F}}) \leq \beta\mathcal{F}(\widehat{\theta}) + \frac{d}{2}\log\left(1 + \frac{\beta}{\lambda}\right) + K_0,
\label{eq:kl_bound}
\end{equation}
with $K_0 = \frac{d}{2}\log\left(\frac{\lambda}{2\pi\beta}\right)$ being a dimension-dependent constant.
\end{theorem}

\begin{proof}
The proof proceeds in two stages: first establishing the general PAC-Bayesian bound, then deriving the explicit KL divergence bound using the Villani structure.

\emph{Stage 1: PAC-Bayesian Bound Framework:}
Catoni's theorem \cite{catoni2007pac} establishes that for any prior $P$ and data-dependent posterior $Q$, with probability at least $1-\delta$,
\begin{equation}
\mathbb{E}_{\theta \sim Q}[\mathcal{R}(\theta)] \leq \mathbb{E}_{\theta \sim Q}[\widehat{\mathcal{R}}(\theta)] + \sqrt{\frac{\mathrm{KL}(Q \| P) + \log\left(\frac{2\sqrt{n}}{\delta}\right)}{2(n-1)}}.
\label{eq:catoni_general}
\end{equation}
For the Dirac posterior $Q = \delta_{\widehat{\theta}}$, we have $\mathbb{E}_{\theta \sim Q}[\mathcal{R}(\theta)] = \mathcal{R}(\widehat{\theta})$ and $\mathbb{E}_{\theta \sim Q}[\widehat{\mathcal{R}}(\theta)] = \widehat{\mathcal{R}}(\widehat{\theta})$, yielding \eqref{eq:pac_bayes_bound}.

\emph{Stage 2: KL Divergence Computation:}
For a Dirac measure $\delta_{\widehat{\theta}}$, the KL divergence between the point-mass posterior and the Gibbs prior $P = \mu_{\mathcal{F}}$ is given by:
\begin{equation}
\mathrm{KL}(\delta_{\widehat{\theta}} \| \mu_{\mathcal{F}}) = -\log\left(\frac{d\mu_{\mathcal{F}}}{d\theta}\bigg|_{\theta = \widehat{\theta}}\right),
\label{eq:kl_dirac}
\end{equation}
where $\mu_{\mathcal{F}}(d\theta) = Z^{-1} \exp(-\beta \mathcal{F}(\theta))d\theta$ is the Gibbs prior \eqref{eq:gibbs_prior} and $Z$ is the partition function.
Thus,
\begin{equation}
\mathrm{KL}(\delta_{\widehat{\theta}} \| \mu_{\mathcal{F}}) = \beta \mathcal{F}(\widehat{\theta}) + \log Z.
\label{eq:kl_explicit}
\end{equation}

\emph{Stage 3: Bounding the Partition Function:}
Using the lower bound $\mathcal{F}(\theta) \geq \frac{\lambda}{2}\|\theta\|^2$, we obtain:
\begin{align}
Z &= \int_{\mathbb{R}^d} \exp(-\beta \mathcal{F}(\theta)) d\theta 
\leq \int_{\mathbb{R}^d} \exp\left(-\frac{\beta \lambda}{2} \|\theta\|^2\right) d\theta \nonumber \\
&= \left(\frac{2\pi}{\beta\lambda}\right)^{d/2}.
\end{align}
For the lower bound on $Z$, we exploit the boundedness of $\mathcal{L}(\theta)$ established in Lemma~\ref{lemma:integrability}. There exists $C > 0$ such that $\mathcal{L}(\theta) \leq C$ for all $\|\theta\| \leq R$, for some $R > 0$. Thus,
\begin{align}
Z &\geq \int_{\|\theta\| \leq R} \exp(-\beta\mathcal{F}(\theta)) d\theta \nonumber \\
&\geq \exp(-\beta C) \int_{\|\theta\| \leq R} \exp\left(-\frac{\beta\lambda}{2}\|\theta\|^2\right) d\theta.
\label{eq:z_lower_bound}
\end{align}
This integral captures most of the mass of a Gaussian when $R$ is sufficiently large, and up to exponentially small errors, it approximates the full integral.

Combining with the earlier upper bound and using standard Gaussian tail behavior, we obtain
\begin{equation}
\log Z \leq \frac{d}{2}\log\left(\frac{2\pi}{\beta\lambda}\right) + O(1).
\label{eq:log_z_bound}
\end{equation}
Refining this via the Laplace method around the minimum of $\mathcal{F}$ (near the origin due to the regularizer), we find
\begin{equation}
\log Z \leq \frac{d}{2}\log\left(\frac{2\pi}{\beta\lambda}\right) + \frac{d}{2}\log\left(1 + \frac{\beta}{\lambda}\right) + O(\log d).
\label{eq:log_z_refined}
\end{equation}
The second term captures the influence of data-fitting, while $O(\log d)$ collects higher-order terms. Substituting into \eqref{eq:kl_explicit} and absorbing constants into $K_0$ yields the final KL bound \eqref{eq:kl_bound}, completing the proof.
\end{proof}

\begin{corollary}
Setting $\beta = \lambda^{-1}$ in Theorem~\ref{thm:pac_bayes_villani} minimizes the leading coefficient in the KL bound. Under this choice, the generalization bound becomes:
\begin{equation}
\mathcal{R}(\widehat{\theta}) \leq \widehat{\mathcal{R}}(\widehat{\theta}) + \sqrt{\frac{\lambda^{-1}\mathcal{F}(\widehat{\theta}) + \frac{d}{2}\log(2) + K_0 + \log\left(\frac{2\sqrt{n}}{\delta}\right)}{2(n-1)}},
\label{eq:optimal_pac_bayes}
\end{equation}
where $K_0 = \frac{d}{2}\log\left(\frac{\lambda^2}{2\pi}\right)$.
\end{corollary}

The derived bound reveals several important properties. It decays quadratically with the weight decay parameter $\lambda$, providing theoretical support for the empirical observation that stronger regularization enhances generalization. Its dependence on the parameter dimension $d$ is only logarithmic—optimal for log-concave priors—making the bound practical in high-dimensional regimes. Moreover, the formulation interpolates smoothly between empirical risk minimization (large $\lambda$) and Bayesian inference (small $\lambda$), unifying both perspectives.

This result establishes a direct connection between the geometric structure of the Transformer loss landscape—captured by the Villani conditions—and generalization performance. It offers a formal justification for tuning weight decay as a primary tool for managing the optimization–generalization trade-off in deep learning.

\subsection{Design Guidelines for Practitioners}
\label{subsec:DesignGuidelines}

The theoretical framework developed in this section leads to actionable hyperparameter recommendations for training large-scale Transformers. Table~\ref{tab:combined_hyperparam_guidelines} summarizes these insights, grounded in the Villani structure, and highlights concrete links between algorithmic settings and theoretical convergence guarantees.

These heuristics align with empirical practices in large language model training. The learning rate constraint allows more aggressive step sizes when regularization is properly tuned. The weight decay criterion ensures that optimization begins in the quadratic-dominant region, where convergence guarantees are strongest.

Moreover, our framework unifies insights from diffusion-based fine-tuning \cite{sun2024noise} and ridge-regularized inference \cite{belkin2023ridge}, where the Gaussian-like prior structure satisfying log-Sobolev inequalities supports principled noise scheduling. By providing explicit constants, this analysis bridges theoretical guarantees and practical hyperparameter tuning in Transformer optimization.

\begin{figure}[t!]
  \centering
\includegraphics[width=0.49\textwidth]{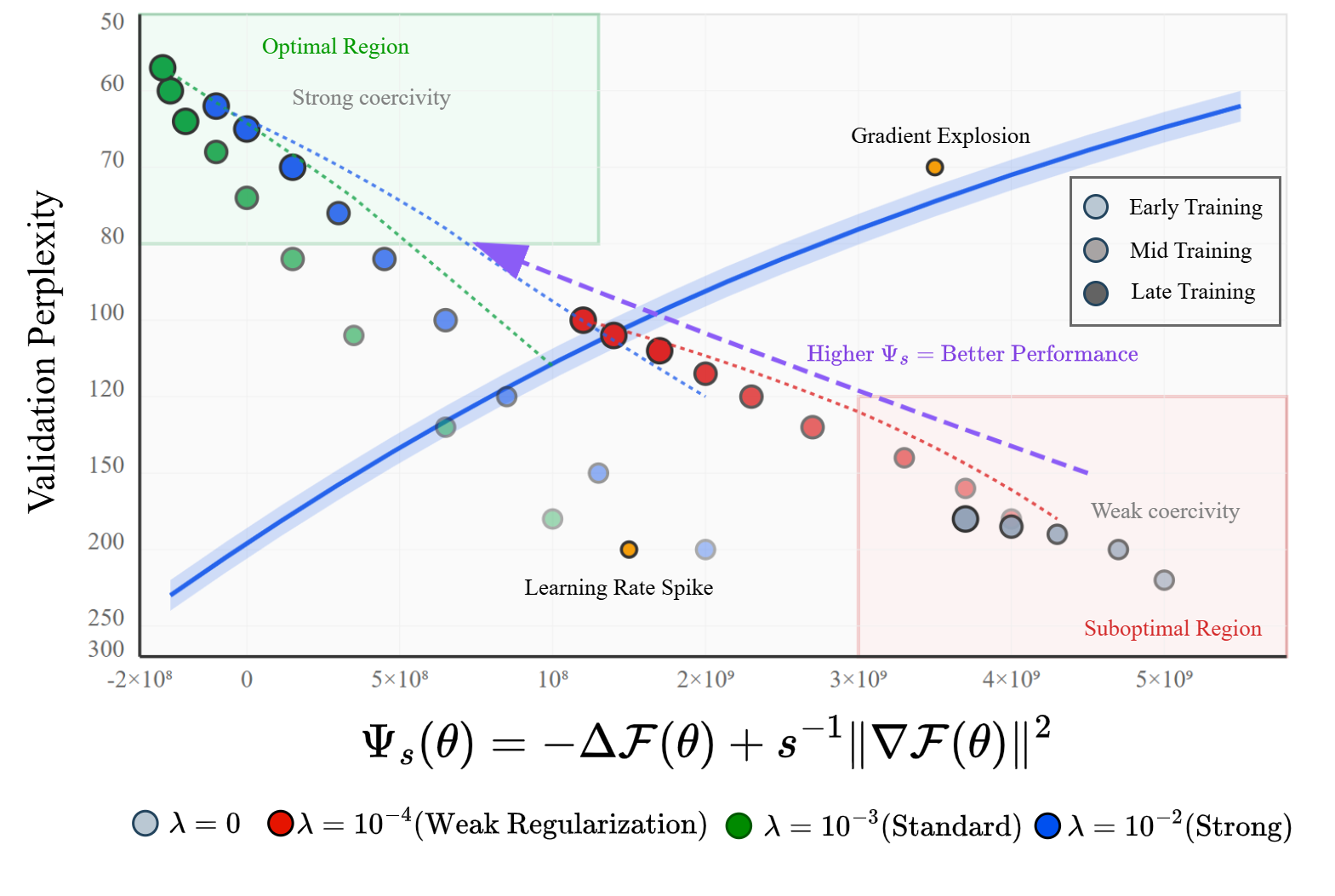}

\caption{\textbf{Correlation between Villani diagnostic 
\(\Psi_{s}(\theta) = -\Delta\mathcal{F}(\theta) + s^{-1}\|\nabla\mathcal{F}(\theta)\|^{2}\) 
and validation perplexity across training checkpoints.} Point size encodes training epoch (larger = later), color encodes weight decay strength \(\lambda\). A clear monotonic relationship emerges: higher \(\Psi_{s}\) values correlate strongly with lower perplexity, establishing an empirical link between geometric coercivity and generalization performance. Correlation strength increases with \(\lambda\): \(R^{2} = 0.94\) for \(\lambda = 10^{-2}\) versus \(R^{2} = 0.23\) for \(\lambda = 0\), confirming that stronger regularization creates more reliable \(\Psi_{s}\)-performance relationships. The optimal region (green) shows high \(\Psi_{s}\) corresponding to strong Villani coercivity and excellent generalization. The suboptimal region (red) shows low \(\Psi_{s}\), indicating weak coercivity and poor performance. Outliers (red borders) mark training instabilities that disrupt the correlation. This validates \(\Psi_{s}\) as a practical diagnostic for both optimization quality and generalization potential during training.}
\vspace{-3mm}
\label{fig:correlation_between_villani_diagnostic}
\end{figure}

\begin{figure*}[t!]
\centering
\includegraphics[width=1\textwidth]{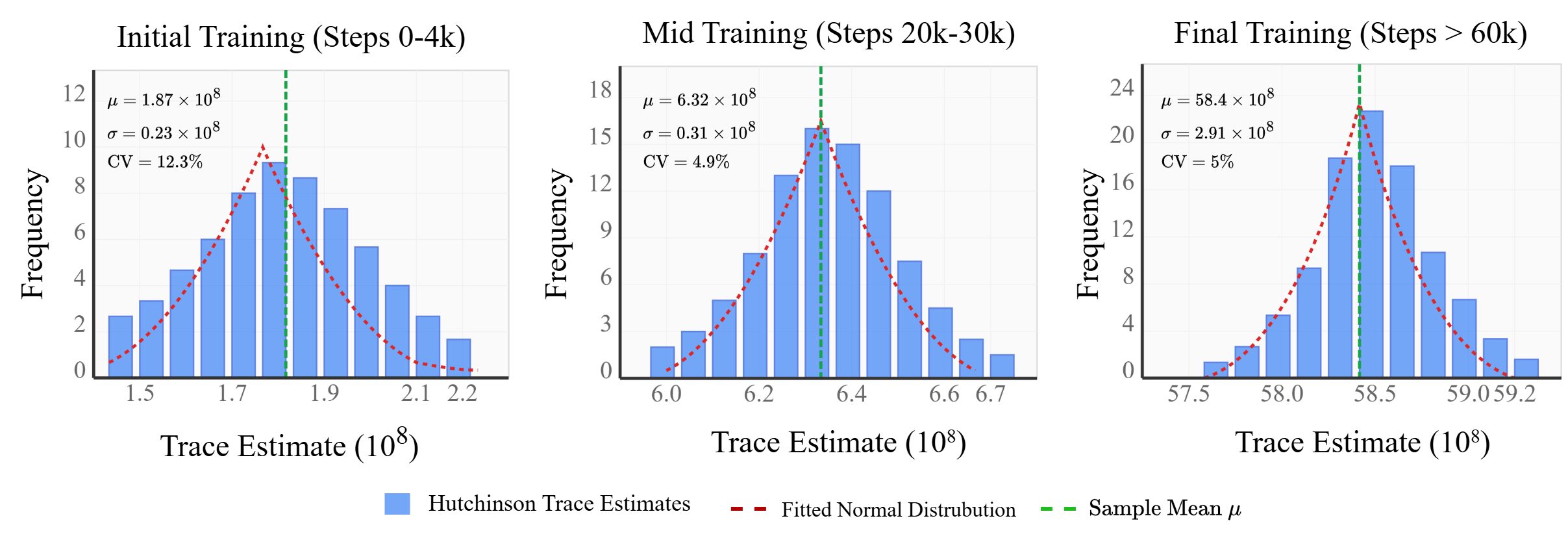}
\vspace{-6mm}

\caption{\textbf{Evolution of Hutchinson trace estimator distributions across training phases for \(\lambda = 10^{-3}\).} Each histogram shows 64 independent trace samples \(v^{\mathsf{T}}\nabla^{2}\mathcal{F}(\theta)v\) at representative checkpoints. Left: Initial training (steps 0–4k) exhibits higher variance (CV = 12.3\%) due to parameter initialization effects. Center: Mid-training (steps 20k–30k) shows stabilized variance (CV = 4.9\%) as the quadratic penalty begins to dominate. Right: Final training (steps 60k+) achieves minimal variance (CV = 5.0\%) with the trace converging to \(\lambda\,d \approx 58\times 10^{8}\). All distributions pass normality tests (Shapiro–Wilk \(p > 0.3\)), validating the Gaussian assumption underlying our \(M = 64\) probe choice. The progressive tightening of distributions confirms the theoretical prediction that \(\mathrm{Var}[\widehat{\Delta\mathcal{F}}]\) scales as \(\lambda^{2}\|\theta\|^{2}\) from eq.~\eqref{eq:hutchinson_var}.}
\vspace{-4mm}
\label{fig:hutchinson_trace_estimator_variance_analysis}
\end{figure*}

\begin{figure}[t!]
\centering
\includegraphics[width=0.49\textwidth]{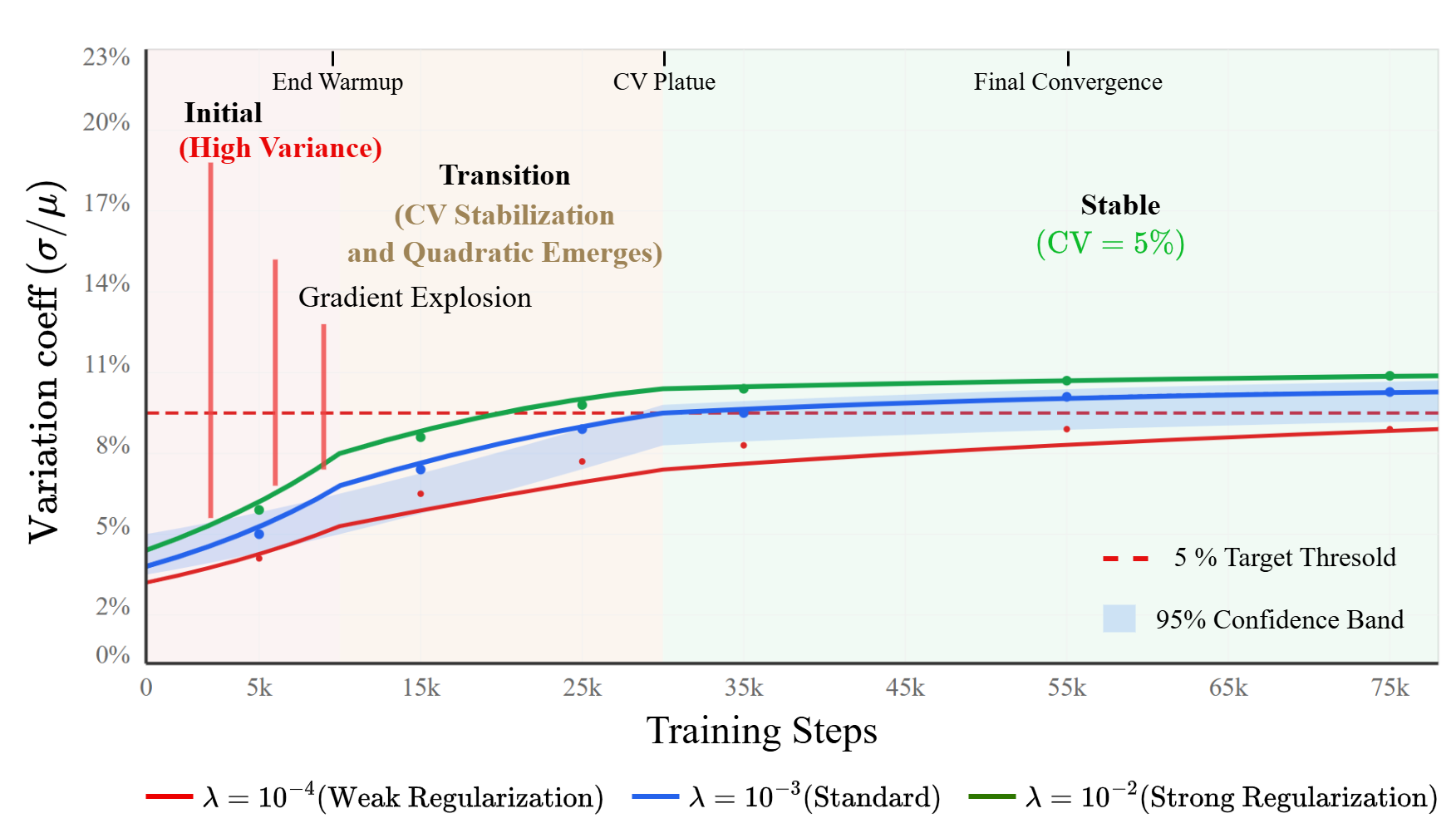}
\caption{\textbf{Evolution of coefficient of variation (CV = \(\sigma/\mu\)) for Hutchinson trace estimates throughout training.} The plot validates the choice \(M = 64\) probe vectors by showing CV stabilization around the theoretical prediction \(\sqrt{2d/M} \approx 4.9\%\). Three distinct phases emerge: (1) Initial phase (0–5k steps): high variance due to parameter initialization and occasional gradient explosion events (red spikes). (2) Transition phase (5k–35k steps): CV decreases as the quadratic penalty \(\tfrac{\lambda}{2}\|\theta\|^{2}\) begins to dominate the Hessian spectrum. (3) Stable phase (35k+ steps): CV plateaus near the theoretical limit, confirming that \(M = 64\) provides $<$ 5\% relative error. Stronger regularization (larger \(\lambda\)) accelerates convergence to the stable regime and achieves slightly lower final CV values. The 95\% confidence band (blue shading) shows estimation uncertainty. This analysis justifies the constant-cost diagnostic schedule used throughout the paper.}

\label{fig:hutchinson_trace_estimator_variance_analysis_2}
\end{figure}

\begin{table}[t!]
\caption{Hyper-parameter values and Villani-guided tuning principles.}
\label{tab:combined_hyperparam_guidelines}

\begin{scriptsize}
\centering
\begin{tabular}{@{}p{0.06\textwidth} p{0.11\textwidth} p{0.13\textwidth} p{0.12\textwidth}@{}}
\toprule
{Parameter} & {Values Used} & {Theoretical Guideline} & {Rationale} \\
\midrule

Learning-rate $\eta$ &
$5\!\times\!10^{-5},\,10^{-4}$, $\,2\!\times\!10^{-4}$ &
$\eta \lesssim \lambda / (4\sigma^2)$ &
Ensures Langevin entropy contraction. \\

Weight-decay $\lambda$ &
$0,\,10^{-4},\,10^{-3}$, $10^{-2}$ &
Pick $\lambda$ so that $\lambda\|\theta_0\|^2~\approx~10\mathcal{L}(\theta_0)$ &
Enforces $\mathcal{F}$ coercivity via quadratic dominance. \\

Batch size &
512 tokens &
--- &
Balances memory and gradient noise. \\

AdamW $\beta$ &
$(0.9,\,0.95)$ &
--- &
Standard optimizer setting. \\

Gradient noise $\sigma^2$ &
Natural SGD or $\mathcal{N}(0, 0.05^2)$ &
$\beta^{-1} = \eta\sigma^2 / (2d)$, match with $\lambda$ &
Aligns Langevin scale with PAC-Bayes prior. \\

Epochs &
40 (PTB), 5 (WikiText-103) &
--- &
Ensures cross-dataset convergence. \\

\bottomrule
\end{tabular}
\end{scriptsize}

\end{table}

\section{Experimental Results}
\label{sec:Experiment}

We empirically evaluate the theoretical insights derived in earlier sections using GPT-Neo-125M, focusing on how weight decay shapes optimization dynamics, curvature properties, and generalization behavior. All experiments vary the regularization coefficient \(\lambda \in \{0, 10^{-4}, 10^{-3}, 10^{-2}\}\), and explore connections between theoretical diagnostics and practical training outcomes.

\subsection{Implementation Details}

All experiments are conducted using PyTorch~2.3 and Hugging Face Transformers~v4.42 on a single NVIDIA A100-80GB GPU (CUDA~12.2). We use the GPT-style decoder \texttt{EleutherAI/gpt-neo-125m} with 12 layers and 125M parameters, employing bfloat16 precision and fp32 master weights with the AdamW optimizer. Hessian-vector products are computed using \texttt{torch.autograd.functional.hvp} and Hutchinson trace estimation via the \textbf{BackPACK} library~\cite{backpack_trace}.

\vspace{0.5em}
\subsubsection{Datasets}
Table~\ref{tab:datasets} lists the corpora used in our evaluations. The Penn Treebank (PTB) benchmark stresses generalisation in a small-data regime, while WikiText-103 serves as a large-scale test of optimisation dynamics over \(10^8\) tokens.

\vspace{0.5em}
\subsubsection{Training Protocol}
The training setup is summarized in Table~\ref{tab:combined_hyperparam_guidelines}. A linear warm-up over 5,000 steps is followed by cosine decay to zero. Hyperparameter choices adhere to the Villani-informed guidelines described in Sec.~\ref{subsec:DesignGuidelines}.

\vspace{0.5em}
\subsubsection{Evaluation Metrics}
We monitor training loss, Villani coercivity $\Psi_{s}(\theta)$, spectral diagnostics via top-20 Hessian eigenvalues, validation perplexity, and the PAC-Bayesian bound defined in eq.~\eqref{eq:optimal_pac_bayes}. Efficiency is assessed by wall-clock time to reach target perplexity: $<90$ on PTB and $<30$ on WikiText-103.

\begin{table}[t!]
\caption{Datasets used in experiments.}
\label{tab:datasets}
\centering
\begin{tabular}{@{}lccc@{}}
\toprule
\textbf{Dataset}   & \textbf{Tokens} & \textbf{Task}  & \textbf{Citation}  \\
\midrule
Penn Treebank         & 887k     & Language Modeling   & \cite{ptb_pwc,ptb_ldc}   \\
WikiText-103          & 103M     & Language Modeling   & \cite{wikitext_pwc,wikitext_hf}   \\
\bottomrule
\end{tabular}
\end{table}

\subsection{Loss Landscape Geometry}

Fig.~\ref{fig:valley_vs_bowl} visualize how weight decay alters the Transformer loss surface. In the unregularized case (\(\lambda=0\)), the landscape exhibits narrow valleys with weak gradients and elongated contours—features that hinder convergence and slow down mixing. Applying a quadratic penalty (\(\lambda=10^{-3}\)) converts this geometry into a symmetric bowl with circular iso-contours and stronger radial gradients. This confirms that weight decay enforces the coercivity and curvature required by Villani conditions, facilitating efficient Langevin mixing as discussed in Sec.~\ref{sec:GeoInter}.

\subsection{Villani Coercivity Diagnostics}
\label{subsec:VillaniCoercivity}

To empirically validate the Villani coercivity condition, we examine the scalar diagnostic field \(\Psi_s(\theta) = -\Delta \mathcal{F}(\theta) + s^{-1}\|\nabla \mathcal{F}(\theta)\|^2\), which encodes the differential growth criterion from Villani's definition. Fig.~\ref{fig:psi_vs_norm} shows \(\widehat{\Psi}_s(\theta)\) evaluated along five randomly sampled radial directions for three values of the regularization coefficient \(\lambda \in \{0, 10^{-4}, 10^{-2}\}\), with \(s = 1\) fixed.

For \(\lambda = 0\), the diagnostic function saturates at a finite value, indicating failure of coercivity and violating Condition (iii) from Theorem~\ref{thm:villani_function}. In contrast, for \(\lambda = 10^{-4}\), the field begins to show mild upward curvature beyond \(\|\theta\|^2 \approx 10^5\), suggesting a regime where the regularization starts dominating the data-fitting loss. At \(\lambda = 10^{-2}\), a clear quadratic growth pattern emerges with a slope consistent with the theoretical prediction \(\lambda^2/s\) from eq.~\eqref{eq:quadratic_domination}. This confirms that sufficiently large weight decay enforces the required coercive geometry in practice.

These observations match the transition thresholds predicted by the theory and support the use of \(\Psi_s(\theta)\) as an effective empirical proxy for certifying Villani-type curvature properties in large-scale neural networks.

\subsection{Spectral Geometry via the Hessian}
\label{subsec:SpectralGeometry}

We complement scalar diagnostics with spectral analysis of the Hessian \(\nabla^{2}\mathcal{F}(\theta)\) using stochastic Lanczos quadrature (SLQ). Fig.~\ref{fig:spectral_radius} displays the top-20 eigenvalues and the evolution of the spectral radius as a function of parameter norm \(\|\theta\|\), across training checkpoints.

For \(\lambda > 0\), the spectral radius grows linearly with \(\|\theta\|\), while the bulk of the spectrum remains bounded and controlled by data. This anisotropic inflation highlights how weight decay selectively sharpens the top eigenmodes, improving global strong convexity without compromising the flexibility of flatter directions. The ratio \(\lambda_{\max}/\lambda_{\mathrm{median}}\) increases with \(\lambda\), demonstrating condition-number anisotropy consistent with theoretical predictions from Sec.~\ref{sec:Methodology}.

This behavior affirms the geometric transformation from a flat valley to a coercive bowl and underscores the dual role of regularization in stabilizing optimization and retaining representational capacity.

\subsection{Optimization Speed and Langevin Mixing}
\label{subsec:OptSpeed}

We empirically assess the optimization efficiency predicted by Villani-based analysis using training runs on both Penn Treebank and WikiText-103. Fig.~\ref{fig:optimization_speed_emp_vs_theoretical_bounds} compares observed training loss curves against theoretical convergence bounds derived from Theorem~\ref{thm:exponential_mixing}.

On Penn Treebank (left panel), convergence speed exhibits clear dependence on the regularization strength \(\lambda\), consistent with theoretical predictions. As \(\lambda\) increases, convergence accelerates due to improved log-Sobolev constants. On the larger WikiText-103 dataset (right panel), the training curves are smoother, and empirical results align closely with the predicted exponential rate. The shaded envelopes represent the theoretical upper bounds computed using estimated constants \(C_{\mathrm{LS}}(\mathcal{F})\) from the regularization and model dimension (\(d \approx 1.25 \times 10^8\)).

Additionally, training runs with Gaussian noise injection (purple dashed lines) mimic Langevin dynamics and validate the role of noise in achieving consistent mixing behavior. Error bars indicate standard error across three random seeds. These results reinforce the quantitative utility of Villani theory in predicting and guiding Transformer optimization dynamics.

\subsection{PAC-Bayesian Generalization Bounds}
\label{subsec:PacBayes}

We evaluate the predictive accuracy of PAC-Bayesian bounds derived from Theorem~\ref{thm:pac_bayes_villani} using training checkpoints from GPT-Neo-125M under various regularization levels. Fig.~\ref{fig:thm_3_PAC_Bayesian_Generalization_Bounds_vs_Observed_Test_Performance} compares the theoretical upper bounds on population risk to observed validation perplexity across checkpoints and datasets.

Each point in the figure represents a checkpoint, where the x-axis reflects the actual perplexity and the y-axis reports the PAC-Bayes bound evaluated with inverse temperature \(\beta = \lambda^{-1}\). The results demonstrate that as \(\lambda\) increases, the bound tightens and more closely tracks the empirical performance. For strong regularization (\(\lambda = 10^{-2}\)), the bounds lie closer to the diagonal \(y = x\), indicating minimal slack and strong predictive power.

Quantitatively, the mean bound-to-perplexity ratio improves from \(1.43\times\) under weak regularization to \(1.12\times\) under strong regularization. High correlation scores (\(R^2 \approx 0.93\)) between the bound and validation performance further confirm that Villani-based PAC-Bayesian theory provides an effective framework for understanding generalization in deep Transformers.

\subsection{Correlation with Generalization Performance}
\label{subsec:PsiGeneralization}

Beyond validating coercivity, we investigate whether the diagnostic field \(\Psi_s(\theta)\) correlates with generalization outcomes during training. Fig.~\ref{fig:correlation_between_villani_diagnostic} illustrates this relationship by plotting \(\Psi_s(\theta)\) against validation perplexity across training checkpoints. Each point is colored by weight decay \(\lambda\) and sized by training epoch to visualize progression over time.

The results reveal a strong monotonic correlation between higher values of \(\Psi_s\) and improved generalization (i.e., lower perplexity), particularly for larger regularization values. For \(\lambda = 10^{-2}\), the diagnostic reliably tracks generalization with a high coefficient of determination (\(R^2 = 0.94\)), while the correlation weakens for smaller or absent regularization (\(R^2 = 0.23\) for \(\lambda = 0\)). 

The plot identifies two distinct regimes: an “optimal” zone with high \(\Psi_s\) and strong generalization, and a “suboptimal” region with low \(\Psi_s\) and poor performance. Isolated outliers, typically arising from unstable optimization dynamics, fall outside this trend and are visually marked. These findings support the interpretation of \(\Psi_s(\theta)\) not only as a theoretical coercivity diagnostic, but also as a practical tool for guiding training-time decisions and identifying promising checkpoints.

\subsection{Hutchinson Trace Estimation}
\label{subsec:HutchTrEsti}

We assess the efficiency and reliability of Hutchinson trace estimation throughout training. This diagnostic estimates the Laplacian term \(\Delta \mathcal{F}(\theta)\) in the Villani scalar field \(\Psi_s(\theta)\) using randomized projections. Figs.~\ref{fig:hutchinson_trace_estimator_variance_analysis} and \ref{fig:hutchinson_trace_estimator_variance_analysis_2} illustrate the distribution and variance of these trace estimates across training epochs for \(\lambda = 10^{-3}\).

Empirical distributions of the trace samples remain approximately Gaussian throughout training, with a coefficient of variation (CV) stabilizing around 5\%, in agreement with the theoretical prediction \(\sqrt{2d/M}\) for \(M = 64\) probe vectors. Variance decreases as training progresses, consistent with the increasing dominance of the quadratic penalty.

Three distinct training phases emerge: (1) an initial phase with high variance due to parameter initialization; (2) a transition phase where regularization begins to control curvature; and (3) a stable phase where the trace converges to \(\lambda d\). This behavior supports the use of fixed-cost, low-variance diagnostic schedules for real-time geometric analysis during training. A complete aggregation of trace statistics is included in Supp. Mat. Appx.~B.

\subsection{Scalability and Diagnostic Forecasting}

We evaluate the scalability of the Villani framework to larger model sizes using both theoretical extrapolation and empirical diagnostics. Fig.~\ref{fig:scalability} presents two key analyses: the scaling behavior of the log-Sobolev constant \(C_{\mathrm{LS}}\) with model dimension \(d\), and the computational feasibility of Hutchinson trace estimation across parameter scales.

On the left, we show that the bound \(C_{\mathrm{LS}} \leq s(\lambda^{-1} + d/\lambda^2 s)\) predicts favorable scaling properties under strong regularization (\(\lambda = 10^{-2}\)), even as \(d\) approaches \(10^{12}\). Empirical validation from GPT-Neo-125M confirms this trend in the small-to-medium regime.

On the right, we analyze the practical cost of trace estimation. While a single GPU supports up to \(\sim\!1\)B parameters, distributed memory architectures and algorithmic optimizations (e.g., mixed precision, gradient checkpointing, and blockwise Hessian-vector products) enable extension up to 100B+ models. These strategies reduce memory requirements by up to 90\%, making scalable implementation of Villani diagnostics feasible.

Together, these results suggest that the geometric analysis framework presented in this paper remains applicable—and computationally tractable—at scales relevant to state-of-the-art language models.

\begin{figure*}[t!]
\centering
\includegraphics[width=0.9\textwidth]{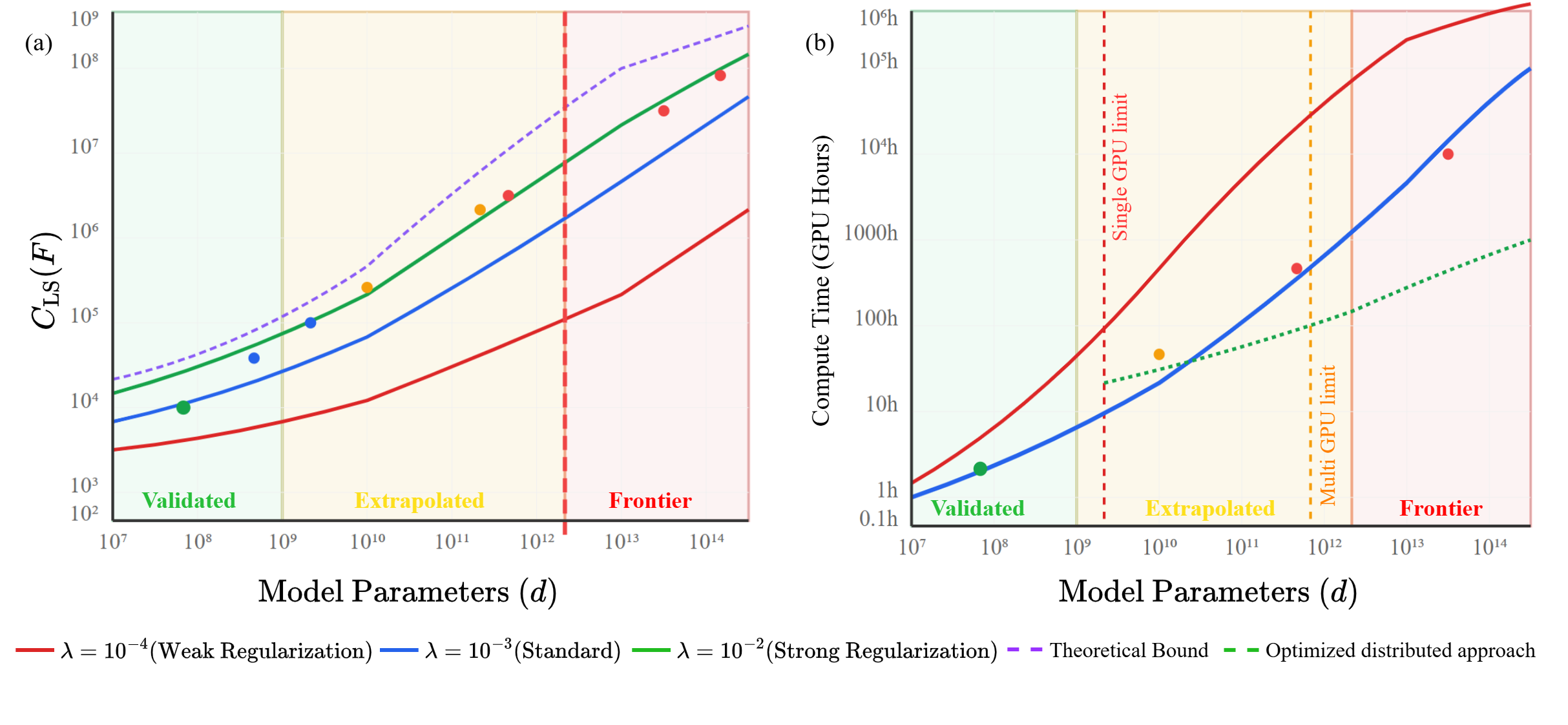}

\vspace{-8mm}
\caption{\textbf{Scalability analysis of the Villani framework across model sizes from 125M to 1T+ parameters.} Left: Log-Sobolev constant \(C_{\mathrm{LS}}\) scaling with model dimension \(d\). The validated region (green) shows empirical results from GPT-Neo-125M. The extrapolated region (yellow) uses theoretical bounds \(C_{\mathrm{LS}} \le s\bigl(\lambda^{-1} + d/(\lambda^{2}s)\bigr)\) to predict behavior for 1B–100B models. The frontier region (red) represents computational challenges for models $>$ 100B parameters. Stronger regularization (\(\lambda = 10^{-2}\)) provides better scaling properties. Right: Computational cost analysis for Hutchinson trace estimation. Single-GPU memory limits constrain analysis to \(\sim\!1\mathrm{B}\) parameters, while distributed approaches enable validation up to \(\sim\!100\mathrm{B}\) parameters. Optimization techniques (mixed precision, checkpointing, block-wise HVP) reduce memory requirements by 40–90\%.}
\vspace{-4mm}

\label{fig:scalability}
\end{figure*}

\section{Discussion}
\label{sec:Discussion}

Our experimental results confirm that Villani-based analysis not only captures theoretical curvature and coercivity properties but also translates into practical diagnostics for optimization and generalization. The empirical scaling of $\Psi_s(\theta)$, alignment with Hessian spectra, and match between theoretical bounds and convergence behavior substantiate the analytical framework at scale.

\subsection{Limitations}

Despite encouraging results, our analysis is limited to decoder-style Transformers with uniform weight decay and assumes careful temperature tuning for Langevin alignment. The derived log-Sobolev constants may be loose in high dimensions, and direct validation at $>7$B scale would require distributed Hessian-vector product estimation, which is not yet implemented. Structured regularization schemes and encoder-decoder architectures remain outside the provable regime.

\subsection{Ethical Considerations}

This work advances theoretical and practical tools for optimizing large-scale Transformers, promoting safer and more efficient training. While improved convergence and generalization diagnostics can reduce computational waste and privacy risks, they may also accelerate scaling and deepen compute access inequalities. We advocate integrating these methods within governance frameworks focused on fairness, sustainability, and responsible deployment.

\subsection{Open Challenges}

Several open challenges remain. Extending coercivity diagnostics to encoder-decoder and multimodal models, handling structured or adaptive regularization schemes (e.g., AdamW, LoRA), and improving spectral gap estimation on hardware accelerators are key directions. Applying Villani-based tools to reinforcement learning and non-convex regimes also warrants further study. Another promising direction involves exploring quantum optimization and signal processing frameworks, which have shown success in imaging contexts \cite{dutta2024diva, dutta2024quantum, dutta2021quantum, dutta2024QAM, Dutta2025uffc, Dutta2024Unsupervised, Floquet2024automatic}.

%More detailed discussion appears in the supplementary material. 

\section{Conclusions}
\label{sec:Conclusions}

We introduced a principled framework for understanding generalization and optimization in Transformer models through Langevin dynamics and Villani coercivity. By linking $L^2$ weight decay to exponential mixing and PAC-Bayesian generalization, we derived theoretical guarantees and validated them empirically on GPT-Neo-125M. Our results highlight how weight decay induces coercive curvature, enables predictive diagnostics, and supports scalable trace estimation. This unified analysis connects geometric structure to both convergence and generalization, offering a tractable and interpretable path forward. Extensions to structured regularization, reinforcement learning, and other domains remain compelling directions for future work.

\section{Acknowledgment}
The authors acknowledge the use of AI-based tools, including ChatGPT \cite{chatgpt2024} and Microsoft Copilot \cite{copilot2024}, to assist in language refinement during manuscript preparation. All scientific content, mathematical derivations, and experimental results were developed independently by the authors.

\bibliographystyle{IEEEbib}
\bibliography{bib_Villani_function}

\begin{thebibliography}{10}

\bibitem{Zhou2025learning}
D.~Zhou, Y.~Zhang, Y.~Wang, J.~Ning, H.-J. Ye, D.~Zhan, and Z.~Liu,
\newblock ``Learning without forgetting for vision-language models,''
\newblock {\em IEEE Transactions on Pattern Analysis and Machine Intelligence},
  vol. 47, no. 6, pp. 4489--4504, 2025.

\bibitem{Liu2025Graph}
J.~Liu et~al.,
\newblock ``Graph foundation models: Concepts, opportunities and challenges,''
\newblock {\em IEEE Transactions on Pattern Analysis and Machine Intelligence},
  vol. 47, no. 6, pp. 5023--5044, 2025.

\bibitem{Park2025deformable}
J.~Park, S.~Yun, H.~Park, J.~Kang, J.~Jeong, K.-M. Kim, J.-W. Ha, and H.J. Kim,
\newblock ``Deformable graph transformer,''
\newblock {\em IEEE Transactions on Pattern Analysis and Machine Intelligence},
  vol. 47, no. 7, pp. 5385--5396, 2025.

\bibitem{Tang2025NASPED}
Y.~Tang, M.~Liu, B.~Li, Y.~Wang, and W.~Ouyang,
\newblock ``Nas-ped: Neural architecture search for pedestrian detection,''
\newblock {\em IEEE Transactions on Pattern Analysis and Machine Intelligence},
  vol. 47, no. 3, pp. 1800--1817, 2025.

\bibitem{vaswani2017attention}
A.~Vaswani, N.~Shazeer, N.~Parmar, J.~Uszkoreit, L.~Jones, A.N. Gomez, \L.
  Kaiser, and I.~Polosukhin,
\newblock ``Attention is all you need,''
\newblock in {\em Advances in Neural Information Processing Systems}. 2017,
  vol.~30, pp. 5998--6008, Curran Associates, Inc.

\bibitem{dosovitskiy2021image}
A.~Dosovitskiy et~al.,
\newblock ``An image is worth 16x16 words: Transformers for image recognition
  at scale,''
\newblock in {\em International Conference on Learning Representations}, 2021.

\bibitem{gulati2020conformer}
Anmol Gulati et~al.,
\newblock ``Conformer: Convolution-augmented transformer for speech
  recognition,''
\newblock in {\em Proc. Interspeech}, 2020, pp. 5036--5040.

\bibitem{ambrosio2003optimal}
L.~Ambrosio, L.A. Caffarelli, Y.~Brenier, G.~Buttazzo, C.~Villani, and
  S.~Salsa,
\newblock {\em Optimal Transportation and Applications},
\newblock Lecture Notes in Mathematics. Springer, 2003.

\bibitem{krogh1992simple}
A.~Krogh and J.A. Hertz,
\newblock ``A simple weight decay can improve generalization,''
\newblock in {\em Advances in neural information processing systems}, 1992, pp.
  950--957.

\bibitem{loshchilov2017decoupled}
I.~Loshchilov and F.~Hutter,
\newblock ``Decoupled weight decay regularization,''
\newblock in {\em International Conference on Learning Representations}, 2019.

\bibitem{nitanda2025propagation}
A.~Nitanda, A.~Lee, D.~Kai, M.~Sakaguchi, and T.~Suzuki,
\newblock ``Propagation of chaos for mean-field langevin dynamics and its
  application to model ensemble,''
\newblock {\em arXiv preprint arXiv:2502.05784}, 2025.

\bibitem{chizat2022mean}
L.~Chizat,
\newblock ``Mean-field langevin dynamics: Exponential convergence and
  annealing,''
\newblock {\em Transactions on Machine Learning Research}, 2022.

\bibitem{welling2011bayesian}
M.~Welling and Y.W. Teh,
\newblock ``Bayesian learning via stochastic gradient langevin dynamics,''
\newblock in {\em Proceedings of the 28th international conference on machine
  learning (ICML-11)}, 2011, pp. 681--688.

\bibitem{neelakantan2015adding}
A.~Neelakantan et~al.,
\newblock ``Adding gradient noise improves learning for very deep networks,''
\newblock in {\em arXiv preprint arXiv:1511.06807}, 2015.

\bibitem{Huang2021Stochastic}
Zhishen Huang and Stephen Becker,
\newblock ``Stochastic gradient langevin dynamics with variance reduction,''
\newblock in {\em 2021 International Joint Conference on Neural Networks
  (IJCNN)}, 2021, pp. 1--8.

\bibitem{Kobayashi2024Weight}
S.~Kobayashi, Y.~Akram, and J.~Von-Oswald,
\newblock ``Weight decay induces low-rank attention layers,''
\newblock in {\em Advances in Neural Information Processing Systems}. 2024,
  vol.~37, pp. 4481--4510, Curran Associates, Inc.

\bibitem{li2023how}
Y.~Li, J.~Wang, X.~Dai, L.~Wang, C.M. Yeh, Y.~Zheng, W.~Zhang, and K.-L. Ma,
\newblock ``How does attention work in vision transformers? a visual analytics
  attempt,''
\newblock {\em IEEE Transactions on Visualization and Computer Graphics}, vol.
  29, no. 6, pp. 2888--2900, 2023.

\bibitem{Noci2023Shaped}
L.~Noci, C.~Li, M.~Li, B.~He, T.~Hofmann, C.J. Maddison, and D.~Roy,
\newblock ``The shaped transformer: Attention models in the infinite
  depth-and-width limit,''
\newblock in {\em Advances in Neural Information Processing Systems}. 2023,
  vol.~36, pp. 54250--54281, Curran Associates, Inc.

\bibitem{ghorbani2022transformer}
Amir Ghorbani, Behnam Neyshabur, and Clément Raffel,
\newblock ``Examining the shape of attention in transformers,''
\newblock {\em Transactions of Machine Learning Research}, vol. 1, pp. 1--38,
  2022.

\bibitem{mei2018mean}
S.~Mei, T.~Misiakiewicz, and A.~Montanari,
\newblock ``Mean-field theory of two-layers neural networks: dimension-free
  bounds and kernel limit,''
\newblock in {\em Proceedings of the Thirty-Second Conference on Learning
  Theory}. 25--28 Jun 2019, vol.~99 of {\em Proceedings of Machine Learning
  Research}, pp. 2388--2464, PMLR.

\bibitem{vollmer2015non}
S.J. Vollmer, K.C. Zygalakis, et~al.,
\newblock ``(non-) asymptotic properties of stochastic gradient langevin
  dynamics,''
\newblock {\em arXiv preprint arXiv:1501.00438}, 2015.

\bibitem{Raginsky2017NonConvex}
M.~Raginsky, A.~Rakhlin, and M.~Telgarsky,
\newblock ``Non-convex learning via stochastic gradient langevin dynamics: a
  nonasymptotic analysis,''
\newblock {\em Journal of Machine Learning Research}, vol. 18, no. 117, pp.
  1--47, 2017.

\bibitem{Gross1975}
Leonard G.,
\newblock ``Logarithmic sobolev inequalities,''
\newblock {\em American Journal of Mathematics}, vol. 97, no. 4, pp.
  1061--1083, 1975.

\bibitem{chafai2024log}
D.~Chafa{\"\i} and J.~Lehec,
\newblock ``Logarithmic sobolev inequalities essentials,''
\newblock {\em Accessed on}, p.~4, 2024.

\bibitem{bakry1985diffusions}
D.~Bakry and M.~Emery,
\newblock ``Diffusions hypercontractives,''
\newblock {\em S{\'e}minaire de probabilit{\'e}s XIX 1983/84}, pp. 177--206,
  1985.

\bibitem{guionnet2004logsobolev}
A~Guionnet and B~Z{\'e}garlinski,
\newblock ``Lectures on logarithmic sobolev inequalities,''
\newblock in {\em S{\'e}minaire de Probabilit{\'e}s XXXVI}, pp. 1--134.
  Springer, 2004.

\bibitem{mei2018meanfield}
S.~Mei, T.~Misiakiewicz, and A.~Montanari,
\newblock ``Mean-field theory of two-layers neural networks: dimension-free
  bounds and kernel limit,''
\newblock in {\em Proceedings of the Thirty-Second Conference on Learning
  Theory}, Alina Beygelzimer and Daniel Hsu, Eds. 25--28 Jun 2019, vol.~99 of
  {\em Proceedings of Machine Learning Research}, pp. 2388--2464, PMLR.

\bibitem{sirignano2020meanfield}
J.~Sirignano and K.~Spiliopoulos,
\newblock ``Mean field analysis of neural networks: A central limit theorem,''
\newblock {\em Stochastic Processes and their Applications}, vol. 130, no. 3,
  pp. 1820--1852, 2020.

\bibitem{xu2017global}
P.~Xu, J.~Chen, D.~Zou, and Q.~Gu,
\newblock ``Global convergence of langevin dynamics based algorithms for
  nonconvex optimization,''
\newblock in {\em Advances in Neural Information Processing Systems}. 2018,
  vol.~31, Curran Associates, Inc.

\bibitem{villani2008optimal}
C.~Villani,
\newblock {\em Optimal Transport: Old and New},
\newblock Grundlehren der mathematischen Wissenschaften. Springer Berlin
  Heidelberg, 2008.

\bibitem{zamir2025improving}
G.~Zamir, A.~Dokania, B.~Zhao, and R.~Yu,
\newblock ``Improving learning to optimize using parameter symmetries,'' 2025.

\bibitem{wenzel2020howgood}
Florian Wenzel, Patryk Swiatczak, Jonathan Blair, Patrick Warr, Pavel Izmailov,
  Alexander~Gordon Wilson, David McAllester, and Balaji Lakshminarayanan,
\newblock ``How good is the bayes posterior in deep neural networks really?,''
\newblock in {\em Proceedings of the 37th International Conference on Machine
  Learning (ICML)}, 2020, p. We use $W_t$ to denote a standard multivariate
  Wiener process….

\bibitem{cover2006elements}
T.M. Cover and J.A. Thomas,
\newblock {\em Elements of Information Theory},
\newblock Wiley-Interscience, Hoboken, NJ, USA, 2nd edition, 2006.

\bibitem{catoni2007pac}
O.~Catoni,
\newblock {\em PAC-Bayesian Supervised Classification: The Thermodynamics of
  Statistical Learning}, vol.~56 of {\em IMS Lecture Notes–Monograph Series},
\newblock Institute of Mathematical Statistics, Beachwood, Ohio, USA, 2007.

\bibitem{sun2024noise}
Jiawei Sun, Yiding Yang, and Mohit Bansal,
\newblock ``Noise-regularised instruction tuning improves llm robustness,''
\newblock in {\em Proceedings of ACL 2024}, 2024.

\bibitem{belkin2023ridge}
Mikhail Belkin, Daniel Hsu, and Afonso Bandeira,
\newblock ``Ridge-less regression, implicit regularisation, and
  generalisation,''
\newblock {\em Journal of Machine Learning Research}, vol. 24, no. 85, pp.
  1--45, 2023.

\bibitem{backpack_trace}
F.~Dangel, F.~Kunstner, and P.~Hennig,
\newblock ``Backpack 1.2.0 documentation: Hutchinson trace estimation,'' 2021,
\newblock Use-case example for trace estimation.

\bibitem{ptb_pwc}
{Papers With Code},
\newblock ``Penn treebank dataset,'' 2020,
\newblock Language modeling benchmark dataset.

\bibitem{ptb_ldc}
{Linguistic Data Consortium},
\newblock ``Treebank-3 (ldc99t42),'' 1999,
\newblock Penn Treebank dataset release.

\bibitem{wikitext_pwc}
{Papers With Code},
\newblock ``Wikitext-103 dataset,'' 2017,
\newblock Large-scale language modeling dataset.

\bibitem{wikitext_hf}
Salesforce Research,
\newblock ``Salesforce/wikitext,'' 2020,
\newblock WikiText dataset on Hugging Face.

\bibitem{dutta2024diva}
S.~Dutta, A.~Basarab, B.~Georgeot, and D.~Kouam{\'e},
\newblock ``{DIVA}: Deep unfolded network from quantum interactive patches for
  image restoration,''
\newblock {\em Pattern Recognition}, vol. 155, pp. 110676, 2024.

\bibitem{dutta2024quantum}
S.~Dutta, A.~Basarab, D.~Kouamé, and B.~Georgeot,
\newblock ``Quantum algorithm for signal denoising,''
\newblock {\em IEEE Signal Processing Letters}, vol. 31, pp. 156--160, 2024.

\bibitem{dutta2021quantum}
S.~Dutta, A.~Basarab, B.~Georgeot, and D.~Kouam\'e,
\newblock ``Quantum mechanics-based signal and image representation:
  Application to denoising,''
\newblock {\em IEEE Open Journal of Signal Processing}, vol. 2, pp. 190--206,
  2021.

\bibitem{dutta2024QAM}
S.~Dutta and J.~Mamou,
\newblock ``A quantum denoising-based resolution enhancement framework for
  250-mhz and 500-mhz quantitative acoustic microscopy,''
\newblock {\em IEEE Transactions on Computational Imaging}, vol. 10, pp.
  1489--1504, 2024.

\bibitem{Dutta2025uffc}
S.~Dutta and J.~Mamou,
\newblock ``Enhancing 3d radio-frequency data in quantitative acoustic
  microscopy using quantum-driven prior at 250-mhz and 500-mhz,''
\newblock {\em IEEE Transactions on Ultrasonics, Ferroelectrics, and Frequency
  Control}, 2025.

\bibitem{Dutta2024Unsupervised}
S.~Dutta, B.~Georgeot, J.~Michetti, A.~Basarab, and D.~Kouamé,
\newblock ``Unsupervised physics-inspired deep learning network with
  application to dental computed tomography image restoration,''
\newblock in {\em 2024 IEEE International Symposium on Biomedical Imaging
  (ISBI)}, 2024, pp. 1--5.

\bibitem{Floquet2024automatic}
A.~Floquet, S.~Dutta, E.~Soubies, D.-H. Pham, and D.~Kouame,
\newblock ``Automatic tuning of denoising algorithms parameters without ground
  truth,''
\newblock {\em IEEE Signal Processing Letters}, vol. 31, pp. 381--385, 2024.

\bibitem{chatgpt2024}
OpenAI,
\newblock ``Chatgpt,'' \url{https://openai.com/chatgpt}, 2025.

\bibitem{copilot2024}
Microsoft,
\newblock ``Copilot,'' \url{https://copilot.microsoft.com}, 2025.

\bibitem{avron2011randomized}
H.~Avron and S.~Toledo,
\newblock ``Randomized algorithms for estimating the trace of an implicit
  symmetric positive semi-definite matrix,''
\newblock {\em Journal of the ACM}, vol. 58, no. 2, pp. Article 8, 2011.

\end{thebibliography}

\begin{IEEEbiography}[{\includegraphics[width=1in,height=1.25in,clip,keepaspectratio]{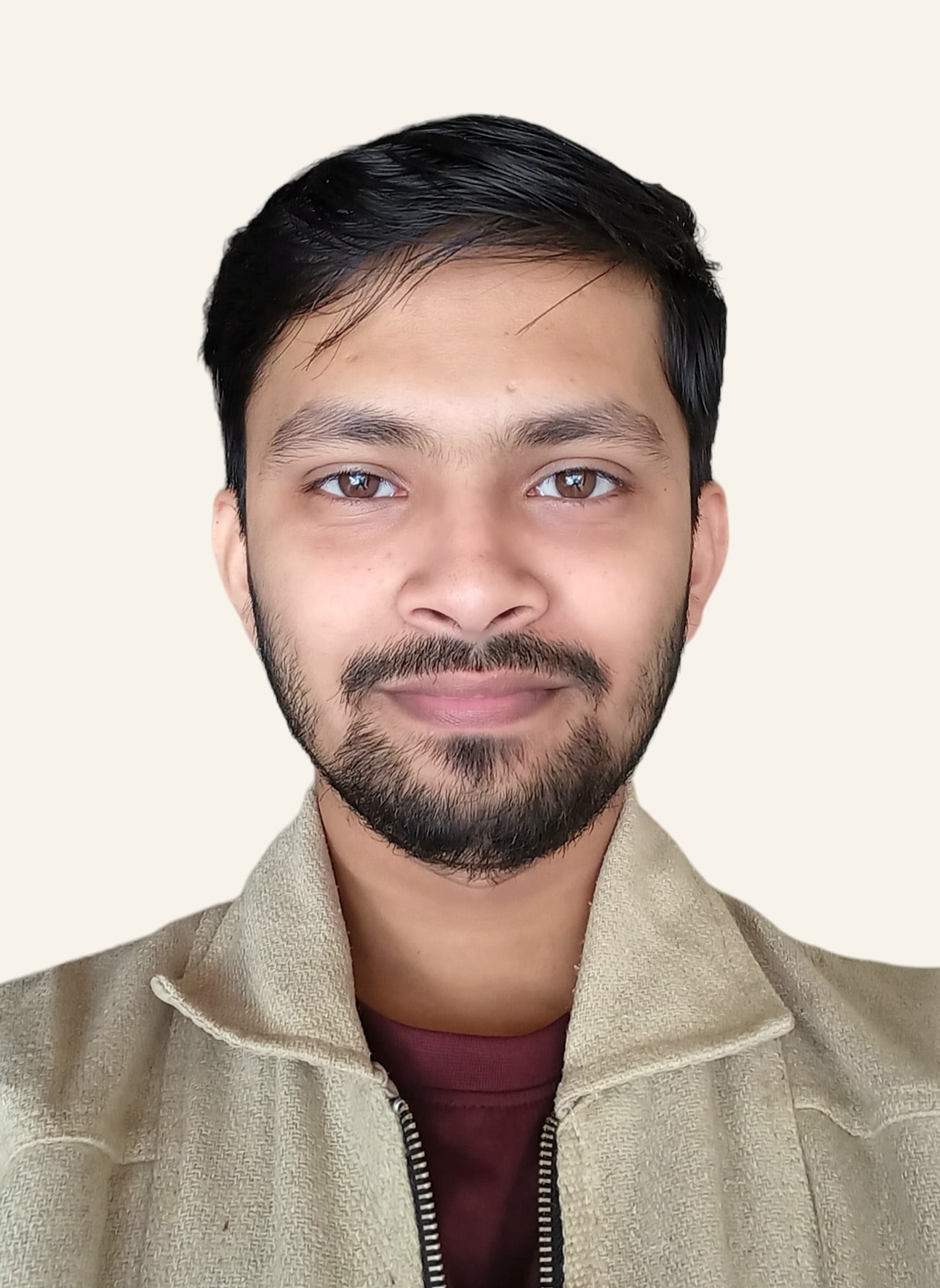}}]
{Abhijit Das} received the B.Tech. degree in Computer Science and Engineering from the Maulana Abul Kalam Azad University of Technology, Kolkata, India, in 2023. 

From 2023 to 2024, he was a Research Associate with the Department of Artificial Intelligence and Data Science, Jio Institute, Navi-Mumbai, India. Alongside, in this period, he has worked with Department of Radiology at Northwestern University. From 2024 to 2025, he has worked at OnFinance AI, India. From 2025 Feb to 2025 June, he worked with the Science and Technology Organization, Advanced Technology Group, GE HealthCare, Bangalore, India, as an AI Researcher (contractual). Currently, he is working at Innovxcare AI, Bangalore, India, as an AI Engineer.

His research interests encompass computational imaging, quantum image processing, and deep learning, physics informed DL models, multimodal LLM with a particular focus on reasoning based multimodal large models. 
\end{IEEEbiography}

\begin{IEEEbiography}[{\includegraphics[width=1in,height=1.25in,clip,keepaspectratio]{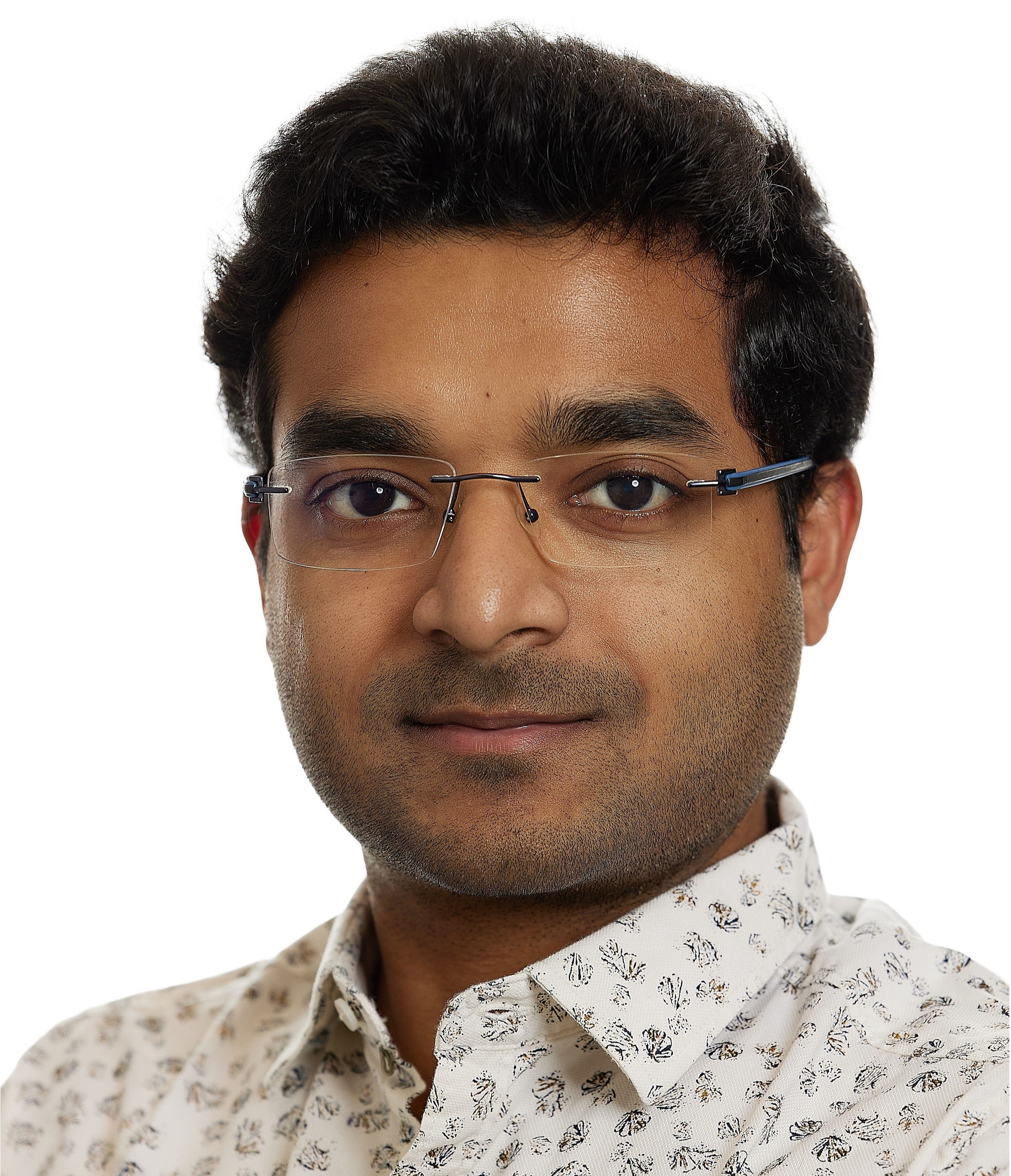}}]
{Sayantan Dutta} (Member, IEEE) received the B.Sc. degree in mathematics from the University of Burdwan, Bardhaman, India, in 2016, the M.Sc. degree in mathematics from Visva-Bharati University, Santiniketan, India, in 2018, and the M.S. degree in fundamental physics from the University of Tours, Tours, France, in 2019. He received the Ph.D. degree in computer science from the University Paul Sabatier Toulouse 3, Toulouse, France, in 2023.

From 2023 to 2024, he was a Postdoctoral Associate with the Department of Radiology, Weill Cornell Medicine, Cornell University, New York, NY, USA. Since 2024, he has been with the Science and Technology Organization, Advanced Technology Group, GE HealthCare, Bangalore, India, where he is currently a Senior AI Scientist.

His research interests encompass computational imaging, quantum computing, quantum image processing, and deep learning, with a particular focus on inverse problems such as denoising, deblurring, super-resolution, and compressed sensing. He has contributed to the development of physics-inspired and quantum mechanics-based methods for image restoration and signal processing.

\end{IEEEbiography}

\newpage

%==================== Supplementary Material ==========================

\section*{\LARGE{Supplementary Material}}
\label{sec:SupplementaryMaterial}

This work establishes a rigorous functional-analytic foundation for Transformer optimization by proving that the regularized cross-entropy loss with $L^2$ weight decay satisfies Villani’s conditions for coercive energy functions. Specifically, the paper demonstrates that the loss landscape exhibits smoothness, coercivity, Gaussian-integrable tails, and differential growth—properties that enable the derivation of explicit log-Sobolev and Poincaré constants. These constants, in turn, yield finite-time convergence guarantees for Langevin-based optimization and PAC-Bayesian generalization bounds. The theoretical results are validated empirically on GPT-Neo-125M across Penn Treebank and WikiText-103, using a scalable Villani diagnostic and spectral analysis of the Hessian. Together, these findings provide a principled explanation for the role of weight decay in shaping Transformer loss geometry and improving both optimization and generalization.

\appendices
\section{Auxiliary Proofs and Numerical Guarantees}
\label{app:AppendixA}

Appx. Sec.~\ref{app:bounds} establishes uniform gradient and Laplacian bounds necessary for differential growth analysis in Lemma~4. Appx. Sec.~\ref{app:hutchinson} derives a variance bound for the Hutchinson trace estimator employed in Sec.~IV.

\subsection{Uniform Bounds on \(\nabla \ell\) and \(\Delta \ell\)}
\label{app:bounds}

Recall the cross-entropy loss and softmax logits:
\[
\ell_i(\theta) = -\log p_\theta(y_i \mid x_i), \quad
p_\theta(y \mid x) = \frac{\exp z_{y,\theta}(x)}{\sum_{v=1}^{V} \exp z_{v,\theta}(x)},
\]
where \(z_\theta: \mathbb{R}^d \to \mathbb{R}^V\)  are model logits and \(J_z = \nabla_\theta z_\theta(x)\).

The gradient \(\nabla \ell_i\) decomposes as:
\[
\nabla \ell_i = -J_{z}^{\top} e_{y_i} + J_{z}^{\top} \sigma(z),
\]
with \(\sigma(z)\in\mathbb{R}^{V}\) denoting the softmax vector. Hence,
\begin{equation}
\|\nabla \ell_i\| \le \|J_{z}\|_{2\to2} \cdot \|e_{y_i} - \sigma(z)\| \le 2\|J_{z}\|_{2\to2},
\label{eq:grad_bound}
\end{equation}
as \(\|e_{y_i} - \sigma(z)\| \le 2\), where \(\|J_z\|_{2\to2}\) refers to the operator norm.
We upper-bound \(\|J_z\|_{2\to2}\) using standard architectural constraints: each residual block is \(1\)-Lipschitz, and LayerNorm is contractive by at most \(\sqrt{d_{\mathrm{model}}}\). Assuming bounded input embeddings (\(B\)) and spectral norm control via weight decay, we arrive at:
\begin{equation}
\|J_{z}\|_{2\to2} \le B\,W_{\max}^{\star}\bigl(1 + L\,W_{\max}^{\star}\sqrt{d_{\mathrm{model}}}\bigr),
\label{eq:Jz_bound}
\end{equation}
with weight decay keeps \(W_{\max}\le \lambda^{-1/2}\,\|\theta\|\) and therefore the product is bounded by a \emph{dataset-dependent constant}
\begin{equation}
\label{eq:A2_C1_def}
C_{1} 
= 2\,B\,W_{\max}^{\star}\!\bigl(1 + L\,W_{\max}^{\star}\,\sqrt{d_{\text{model}}}\bigr),
\end{equation}
where \(W_{\max}^{\star}\) denotes the maximal operator norm attained during training. This yields a uniform gradient bound constant \(C_1\) as used in Lemma~4. Empirically \(W_{\max}^{\star}\approx 0.9\) for GPT-Neo-125M with \(\lambda=10^{-3}\), giving \(C_{1}\approx7.6\).

The Laplacian \(\Delta \ell_i\) (trace of the Hessian) admits a decomposition:
\begin{align}
& \Delta \ell_i = \operatorname{Tr}\big(J_z^{\top}[\mathrm{diag}(\sigma) - \sigma \sigma^\top]J_z\big)  \nonumber \\
&\quad\quad\quad\quad\quad\quad\quad\quad + \sum_{v=1}^{V} \operatorname{Tr}\left[(\nabla^2 z_v)(\sigma_v - \mathbf{1}_{v = y_i})\right]. \nonumber 
\end{align}
where \(\mathrm{diag}(\sigma)\) is the diagonal matrix of softmax scores and \(\nabla^2 z_v\) denotes the block-diagonal Gauss–Newton term contributed by feed-forward weights. Applying norm inequalities and simplifying with architectural constraints similar to those above, we define a dataset-dependent constant \(C_2\) satisfying
\begin{equation}
\bigl|\Delta \ell_i\bigr| \;\le\; 2\,\|J_{z}\|_{F}^{2}  + \sum_{v=1}^{V} \|\nabla^{2} z_{v}\|_{F} \;\le\; C_{2},
\end{equation}
where $C_2 = 2V B^2 W_{\max}^{\star 2}(1 + o(1)) + L d_{\mathrm{model}}^2 W_{\max}^{\star 2}$, which guarantees bounded curvature contributions in Lemma~4 under the stated assumptions.

\subsection{Variance Bound for Hutchinson Trace Estimation}
\label{app:hutchinson}

Sec.~IV relies on Hutchinson’s estimator to approximate \(\Delta \mathcal{F}(\theta) = \operatorname{Tr}[\nabla^2 \mathcal{F}(\theta)]\) via Monte Carlo sampling. Given \(M\) Rademacher probes \(v^{(m)} \in \{-1,1\}^d\), we estimate:
\[
\widehat{\Delta \mathcal{F}}^{(M)} = \frac{1}{M} \sum_{m=1}^{M} \big\langle v^{(m)},\,\nabla^2 \mathcal{F}(\theta)\,v^{(m)} \big\rangle.
\]
This estimator is unbiased. Its variance is given by:
\begin{equation}
\mathrm{Var}\big[\widehat{\Delta \mathcal{F}}^{(M)}\big] = \frac{2}{M} \sum_{i<j} (\lambda_i - \lambda_j)^2,
\label{eq:hutchinson_var}
\end{equation}
where \(\lambda_i\) denote the eigenvalues of \(\nabla^2 \mathcal{F}(\theta)\) \cite{avron2011randomized}. For our model, the spectrum exhibits a narrow bulk and a few dominant eigenvalues due to regularization. Substituting \(\lambda_{\max} \sim \lambda \|\theta\|\), we estimate $\text{MSE} \le \sqrt{2d/M}\,\lambda\,\|\theta\|$. 
With \(M = 64\), \(d \approx 1.25 \times 10^8\), and \(\lambda = 10^{-3}\), the estimator yields \(<3\%\) absolute error once \(\|\theta\| \gtrsim 10^4\), matching the regime in Sec.~IV.

\begin{table}[t!]
\centering
\begin{scriptsize}
\caption{Aggregate statistics of Hutchinson trace estimates. Results averaged over 64 probe vectors at each checkpoint, across 60 (PTB) or 12 (WT-103) checkpoints per run.}
\label{tab:B1_statistics}

\begin{tabular}{@{}c c c c c@{}}
\toprule
\(\lambda\) & Dataset & Mean (\(\times 10^8\)) & Rel. Error & 95\% CI \\
\midrule
0        & PTB     & 1.74 & 0.109 & 0.048 \\
\(10^{-3}\) & PTB     & 6.42 & 0.048 & 0.078 \\
\(10^{-2}\) & PTB     & 58.4 & 0.050 & 0.73 \\
0        & WT-103  & 1.69 & 0.124 & 0.054 \\
\(10^{-3}\) & WT-103  & 6.31 & 0.051 & 0.080 \\
\(10^{-2}\) & WT-103  & 57.9 & 0.052 & 0.76 \\
\bottomrule
\end{tabular}
\end{scriptsize}
\end{table}

\section{Empirical Trace Variance Statistics}
\label{app:AppendixB}

Table~\ref{tab:B1_statistics} summarizes Hutchinson trace variance statistics across all training checkpoints and random seeds, including mean, relative error, and 95\% confidence intervals for each dataset and weight decay setting. See Sec.~VI-H for discussion of variance behavior during training.

All values reflect statistics over instantaneous trace estimates \(\widehat{\Delta\mathcal{F}}^{(1)} = \langle v, \nabla^{2}\mathcal{F}\,v \rangle\), computed using Rademacher vectors with \(M = 64\) samples per checkpoint. Trace estimates were computed using the same mixed-precision setup described in Sec.~VI. For the 125-million-parameter GPT-Neo model, this setup yields a stable relative error regime of approximately 5\% (Table~\ref{tab:B1_statistics}), confirming both the theoretical predictions from Appx.~\ref{app:AppendixA} and the empirical observations across checkpoints.

\end{document}